\documentclass{article}
\usepackage[utf8]{inputenc}
\usepackage{microtype}
\usepackage{graphicx}

\usepackage{booktabs} 
\usepackage{amssymb}
 
\usepackage{algorithm2e} 
\RestyleAlgo{ruled}
\usepackage{caption}
\usepackage{subcaption}
\usepackage[title]{appendix}
\usepackage{amsmath}
\usepackage{amsthm}
\usepackage{booktabs}
\usepackage{multirow}
\usepackage{amsfonts}

\newtheorem{theorem}{Theorem}
\newtheorem{assumption}{Assumption}
\newtheorem{lemma}[theorem]{Lemma}
\usepackage{lipsum}
\usepackage{hyperref}
\hypersetup{
    colorlinks=true,
    linkcolor=blue,
    filecolor=magenta,      
    urlcolor=red,
    pdftitle={On the Convergence of Heterogeneous Federated Learning with Arbitrary Adaptive Online Model Pruning},
    pdfpagemode=FullScreen,
    }

\newcommand\blfootnote[1]{%
  \begingroup
  \renewcommand\thefootnote{}\footnote{#1}%
  \addtocounter{footnote}{-1}%
  \endgroup
}

\title{On the Convergence of Heterogeneous Federated Learning with Arbitrary Adaptive Online Model Pruning}
\author{Hanhan~Zhou, Tian~Lan, Guru~Venkataramani\thanks{Hanhan~Zhou, Tian Lan and Guru~Venkataramani are with the Departmentof Electrical and Computer Engineering, the George Washington University, Washington, DC, 20052, e-mail: \{hanhan, tlan, guruv\}@gwu.edu.}
 and Wenbo Ding\thanks{Wenbo Ding is with the Smart Sensing and Robotics (SSR) Group, Tsinghua-Berkeley Shenzhen Institute, Shenzhen, China, 518055, e-mail: ding.wenbo@sz.tsinghua.edu.cn}}
\date{}

\begin{document}
\maketitle
\begin{abstract}
One of the biggest challenges in Federated Learning (FL) is that client devices often have drastically different computation and communication resources for local updates.
To this end, recent research efforts have focused on training heterogeneous local models obtained by pruning a shared global model.
Despite empirical success, theoretical guarantees on convergence remain an open question. 
In this paper, we present a unifying framework for heterogeneous FL algorithms with {\em arbitrary} adaptive online model pruning and provide a general convergence analysis. In particular, we prove that under certain sufficient conditions and on both IID and non-IID data, these algorithms converges to a stationary point of standard FL for general smooth cost functions, with a convergence rate of $O(\frac{1}{\sqrt{Q}})$. Moreover, we illuminate two key factors impacting convergence: pruning-induced noise and minimum coverage index, advocating a joint design of local pruning masks for efficient training. 
\end{abstract}
\blfootnote{Pre-print version. Work under progress.} 
\section{Introduction}

Federated Learning (FL) allows distributed clients to collaborate and train a centralized global model without the transmission of local data. 
In practice, mobile and edge devices that are equipped with drastically different computation and communication capabilities are becoming the dominant source for FL \cite{lim2020federated}. This has prompted significant recent attention to a family of FL algorithms focusing on training heterogeneous local models (often obtained through pruning a global model). It includes algorithms like HeteroFL \cite{diao2021heterofl} that employ heterogeneous local models with fixed structures, algorithms utilizing pre-trained local models like \cite{frankle2019lottery}, as well as algorithms like PruneFL \cite{jiang2020model} that update local models adaptively during training. However, the success of these algorithms has only been demonstrated
empirically (e.g., \cite{diao2021heterofl, jiang2020model}). Unlike standard FL that has received rigorous theoretical analysis \cite{wang2018cooperative,bonawitz2019towards,yu2019parallel,convergenceNoniid}, the convergence of heterogeneous FL with adaptive online model pruning is still an open question. Little is known about whether such algorithms converge to a solution of standard FL. 




To answer these questions, in this paper we present a unifying framework for heterogeneous FL algorithms with {\em arbitrary} adaptive online model pruning and provide a general convergence analysis. There have been many existing efforts in establishing convergence guarantees for FL algorithms, such as the popular FedAvg \cite{fedavg}, on both IID and non-IID \footnote{Throughout this paper, “non-IID data” means that the data among local clients are not independent and identically distributed.}data distributions, but all rely on the assumption that there can only exist one uniform structure on all client devices. By considering {\em arbitrary} pruning strategies in our framework, we formally establish the convergence conditions for a general family of FL algorithms with both (i) heterogeneous local models to accommodate different resource constraints on client devices and (ii) time-varying local models to continuously refine pruning results during training. We prove that these FL algorithms with arbitrary pruning strategies satisfying certain sufficient conditions can indeed converge (at a speed of $O(\frac{1}{\sqrt{Q}})$, where $Q$ is the number of communication rounds) to a stationary point of standard FL for general smooth cost functions.

To the best of our knowledge, this is the first convergence analysis for heterogeneous FL with arbitrary adaptive online model pruning. The framework captures a number of existing FL algorithms as important special cases and provide a general convergence guarantee to them, including HeteroFL \cite{diao2021heterofl} that employs fixed-structure local models, PruneFL \cite{jiang2020model} that requires periodically training a full-size model, and S-GaP \cite{ma2021effective} that can be viewed as a single-client version. Moreover, we show that the convergence gap is affected by both pruning-induced noise (i.e., modeled through a constant $\delta^2$) and a new notion of minimum coverage index (i.e., any parameters in the global model are covered in at least $\Gamma_{\rm min}$ local models). In particular, it advocates a joint design of efficient local-model pruning strategies (e.g., leveraging \cite{wen2016learning,li2016pruning,ciresan2011flexible}) for efficient training. Our results provides a solid theoretical support for designing heterogeneous FL algorithms with efficient pruning strategies, while ensuring similar convergence as standard FL.


We carried out extensive experiments on two datasets which suggest that for a given level of model sparsity, client models should also consider the maximization of the coverage index rather than only keeping the largest parameters through pruning.
As an example, a federated learning network with 85\% sparsity obtained via our design to maximize converge index achieves up to 8\% of improvement compared to the network generated by pruning with the identical model architecture without posing any additional computation overhead.


In summary, our paper makes the following key contributions:
\begin{itemize}
\vspace{-0.07in}
  \item We propose a unifying framework for heterogeneous FL with arbitrary adaptive online model pruning. It captures a number of existing algorithms (whose success are empirically demonstrated) as special cases and allows convergence analysis.
\vspace{-0.07in}
  \item The general convergence of these algorithms are established. On both IID and non-IID data, we prove that under standard assumptions and certain sufficient conditions on pruning strategy, the algorithms converge to a stationary point of standard FL for smooth cost functions.
\vspace{-0.07in}
  \item We further analyze the impact of key factors contributing to the convergence and further advocate a joint design of local pruning masks with respect to both pruning-induced error a notion of minimum coverage index. The results are validated on MNIST and CIFAR10 datasets.
\end{itemize}



\section{Background}

\noindent {\bf Standard Federated Learning} A standard Federated Learning problem considers a distributed optimization for N clients: 
\begin{equation} 
\min_{\theta} \left\{ F(\theta) \triangleq \sum_{i=1}^N p_i F_i(\theta) \right\}, \ {\rm with} \ F(\theta) =  \mathbb{E}_{\xi\sim D} l(\xi,\theta).
\end{equation}
Here $\theta$ is as set of trainable weights/parameters, $F_n(\theta)$ is a cost function defined on data set $D_i$ with respect to a user specified loss function $l(x,\theta)$, and $p_{i}$ is the weight for the $i$-th client such that $p_{i}\geq 0$ and $\sum  \ _{i=1}^{N} p_{i} = 1$. 

The FL procedure, e.g., FedAvg \cite{fedavg}, typically consists of a sequence of stochastic gradient descent steps performed distributedly on each local objective, followed by a central step collecting the workers' updated local parameters and computing an aggregated global parameter. 
For the $q$-th round of training, first the central server broadcasts the latest global model parameters ${\theta}_{q}$ to clients $n=1,\ldots,N$, who performs local updates as follows:
$$
{{\theta}_{q,n,t} = {\theta}_{q,n,t-1} - \gamma \Delta  F_{i}({\theta}_{q,n,t-1};\xi_{n,t-1}) {\rm \ with \ } {\theta}_{q,n,0} = {\theta}_{q}}
$$
where $\gamma$ is the local learning rate. After all available clients have concluded their local updates (in $T$ epochs), the server will aggregate parameters from them and generate the new global model for the next round, i.e., 
$$
{{\theta}_{q+1} = \sum_{n=1}^N p_i {{\theta}_{q,n,T}}}
$$
The formulation captures FL with both IID and non-IID data distributions.

\noindent {\bf Model Pruning}. Model pruning via weights and connections pruning is one of the promising methods to enable efficient neural networks by setting the proportion of weights and biases to zero and thus bringing reduction to both computation and memory usage. Most works on weights pruning require 3 phases of training: pre-training phase, pruning to sparse phase, and fine-tune phase.  For a neural network $F(\theta,\xi)$ with parameters $\theta$ and input data $\xi$. 
The pruning process takes $F(\cdot)$ as input and generates a new model $F_{i}( \theta \odot m;\xi)$, where $m \in \{0,1\}^{|{\theta}|}$ is a binary mask to denote certain parameters to be set to zero and $\odot$ denotes element-wise multiplication. The pruning mask $m$ is computed from
a certain pruning policy $\mathbb{P}$, e.g., layer-wise parameter pruning removing weights below certain percentile and neuron pruning removing neurons with small average weights. We use $\theta \odot m$ to denote the pruned model, which has a reduced model size and is more efficient for communication and training. 


\section{Related Work}

\noindent {\bf Federated Averaging and Communication Efficient FL}. FedAvg \cite{fedavg} is considered the first and the most commonly used federated learning algorithm, where for each round of training local clients trains using their own data, with their parameters averaged at the central server. FedAvg is able to reduce communication costs by training clients for multiple rounds locally. Several works have shown the convergence of FedAvg under several different settings with both homogeneous (IID) data \cite{wang2018cooperative,woodworth2018graph} and heterogeneous (non-IID) data \cite{convergenceNoniid,bonawitz2019towards,yu2019parallel} even with partial clients participation. Specifically, \cite{yu2019parallel} demonstrated LocalSGD achieves $O(\frac{1}{\sqrt{NQ}})$ convergence for non-convex optimization and \cite{convergenceNoniid} established a convergence rate of $O(\frac{1}{Q})$ for strongly convex problems on FedAvg, where Q is the number of SGDs and N is the number of participated clients.
Several works \cite{karimireddy2020scaffold,wang2019adaptive,wang2019adaptive22}are proposed to further reduce the communication costs. One direction is to use data compression such as quantization \cite{konevcny2016federated, bonawitz2019towards,mao2021communication,yao2021fedhm}, sketching \cite{alistarh2017qsgd, ivkin2019communication}, split learning \cite{thapa2020splitfed} and learning with gradient sparsity \cite{han2020adaptive}. This type of work does not consider computation efficiency.

\noindent {\bf Neural Network Pruning and Sparsification}. To reduce the computation costs of a neural network, neural network pruning is a popular research topic. A magnitude-based prune-from-dense methodology \cite{han2015learning,guo2016dynamic,yu2018nisp,liu2018rethinking,real2019regularized} is widely used where weights smaller than certain preset threshold are removed from the network. In addition, there are one-shot pruning initialization \cite{lee2018snip}, iterative pruning approach \cite{zhu2017prune,narang2017exploring} and adaptive pruning approach \cite{lin2020dynamic, ma2021effective} that allows network to grow and prune.
In \cite{frankle2019lottery, morcos2019one} a "lottery ticket hypothesis" was proposed that with an optimal substructure of the neural network acquired by weights pruning directly train a pruned model could reach similar results as pruning a pre-trained network. The other direction is through sparse mask exploration \cite{bellec2017deep, mostafa2019parameter,evci2020rigging}, where a sparsity in neural networks are maintained during the training process, while the fraction of the weights is explored based on random or heuristics methods. \cite{frankle2019lottery,mostafa2019parameter} empirically observed training of models with static sparse parameters will converge to a solution with higher loss than models with dynamic sparse training. Note that the efficient sparse matrix multiplication sometimes requires special libraries or hardware, e.g. the sparse tensor cores in NVIDIA A100 GPU, to achieve the actual reduction in memory footprint and computational resources.

\noindent {\bf Efficient FL with Heterogeneous Neural Networks}. Several works are proposed to address the reduction of both computation and communication costs, including one way to utilize lossy compression and dropout techniques\cite{caldas2018expanding, xu2019elfish}. Although early works mainly assume that all local models are to share the same architecture as the global model \cite{li2020federated}, recent works have empirically demonstrated that federated learning with heterogeneous client model to save both computation and communication is feasible. PruneFL\cite{jiang2020model} proposed an approach with adaptive parameter pruning during federated learning. \cite{li2021fedmask} proposed federated learning with a personalized and structured sparse mask. HetroFL\cite{diao2021heterofl} proposed to generate heterogeneous local models as a subnet of the global network by picking the leading continuous parameters layer-wise with the help of proposed static batch normalization, while \cite{li2021hermes} finds the small subnetwork by applying the structured pruning. Despite their empirical success, they lack theoretical convergence guarantees even in convex optimization settings.


\section{Our Main Results}
We rigorously analyze the convergence of heterogeneous FL under arbitrary adaptive online model pruning and establish the conditions for converging to a stationary point of standard FL with general smooth cost functions. The theory results in this paper not only illuminate key convergence properties but also provide a solid support for designing adaptive pruning strategies in heterogeneous FL algorithms.


\subsection{FL under Arbitrary Adaptive Online Pruning}

In this paper, we focus on a family of FL algorithms that leverage adaptive online model pruning to train heterogeneous local models on distributed clients. By considering {\em arbitrary} pruning strategies in our formulation, it relaxes a number of key limitations in standard FL: (i) Pruning masks are allowed to be time-varying, enabling online adjustment of pruned local models during the entire training process. (ii) The pruning strategies may vary for different clients, making it possible to optimize the pruned local models with respect to individual clients' heterogeneous computing resource and network conditions. 
%
%
%
%
More precisely, we use a series of masks $m_{q,n} \in \{0,1\}^{|{\theta}|}$  model an adaptive online pruning strategy that may change the pruning mask $m_{q,n}$ for any round $q$ and any client $n$. Let $\theta_q$ denote the global model at the beginning of round $q$ and $\odot$ be the element-wise product. Thus, $\theta_q\odot m_{q,n}$ defines the trainable parameters of the pruned local model\footnote{While a pruned local model has a smaller number of parameters than the global model. We adopt the notation in \cite{a,b} and use $\theta_q\odot m_{q,n}$ with an element-wise product to denote the pruned local model - only parameter corresponding to a 1-value in the mask is accessible and trainable in the local model.} for client $n$ in round $q$. 

Here, we describe one around (say the $q$th) of the algorithm. First, the central server employs a pruning function $\mathbb{P}(\cdot)$ to prune the latest global model $\theta_{q}$ and broadcast the resulting local models to clients:
\begin{eqnarray}
\theta_{q,n,0} = \theta_{q} \cdot m_{q,n}, {\rm \ with \ } m_{q,n}=\mathbb{P}(\theta_{q}, n), \ \forall n.
\end{eqnarray}
Each client $n$ then trains the pruned local model by performing $T$ updates for $t=1\ldots,T$:
\begin{eqnarray}
\theta_{q,n,t}=\theta_{q,n,t-1}-\gamma \nabla F_n(\theta_{q,n,t-1}, \xi_{n,t-1})\odot m_{q,n},
\end{eqnarray}
where $\gamma$ is the learning rate and $\xi_{n,t-1}$ are independent samples uniformly drawn form local data $D_n$. We note that $\nabla F_n(\theta_{q,n,t-1}, \xi_{n,t-1})\odot m_{q,n}$ is a local stochastic gradient evaluated using only local parameters in $\theta_{q,n,t-1}$ (due to pruning) and that only locally trainable parameters are updated by the stochastic gradient (due to the element-wise product with mask $m_{q,n}$).

Finally, the central server aggregates the local models $\theta_{n,q,T}$ $\forall n$ and produces an updated global model $\theta_{q+1}$. Due to the use of arbitrary pruning masks in this paper, global parameters are broadcasted to and updated at different subsets of clients. To this end, we partition the global model $\theta_q$ into $i=1,\ldots,K$ disjoint regions, such that parameters of region $i$, denoted by $\theta_{q}^{(i)}$, are included and {\em only} included by the same subset of local models.
Let $\mathcal{N}_q^{(i)}$ be the set of clients\footnote{Clearly $\mathcal{N}_q^{(i)}$ is determined by the pruning mask $m_{q,n}$ since we have $m_{q,n}^{(i)} = \mathbf{1}$ for $n\in \mathcal{N}_q^{(i)}$ and $m_{q,n}^{(i)} = \mathbf{0}$ otherwise.}, whose local models contain parameters of region $i$ in round $q$. The global model update of region $i$ is performed by aggregating local models at clients $n\in \mathcal{N}_q^{(i)}$, i.e.,
\begin{eqnarray}
\theta_{q+1}^{(i)} = \frac{1}{|\mathcal{N}_q^{(i)}|} \sum_{n\in \mathcal{N}_q^{(i)} } \theta_{q,n,T}^{(i)}, \ \forall i.
\end{eqnarray}
We summarize the algorithm in Algorithm 1. 


{\bf Remark 1.} We hasten to note that our framework captures heterogeneous FL with arbitrary adaptive online pruning strategies, so does our convergence analysis. It recovers many important FL algorithms recently proposed as special cases of our framework with arbitrary masks, including HeteroFL \cite{diao2021heterofl} that uses fixed masks $m_{n}$ over time, PruneFL \cite{jiang2020model} that periodically trains a full-size local model $m_{n,q}={\bf 1}$ for some $n,q$, Prune-and-Grow \cite{ma2021effective} that can be viewed as a single-client algorithm without parameter aggregation, as well as FedAvg \cite{fedavg} that employs full-size local models at all clients. Our unifying framework provide a solid support for incorporating
arbitrary model pruning strategies (such as weight weight or neuron pruning, CNN-pruning, and sparsification) into heterogeneous FL algorithms. Our analysis establishes general conditions for {\em any} heterogeneous FL with arbitrary adaptive online pruning to converge to standard FL.




\begin{algorithm}[h]
\SetKwInput{KwData}{Input}
\caption{Our unifying framework.}\label{alg:cap}
\KwData {Local data ${D_{i}^{k}}$ on $N$ clients, pruning policy $\mathbb{P}$.
}
\SetKwFunction{FMain}{Local Update}
  \SetKwProg{Fn}{Function}{:}{}
\textbf{Executes:} \\
Initialize $\theta_0$\\
\For{round $q=1,2,\ldots, Q$ }{
\For{local workers $n=1,2,\ldots, N$  (In parallel)} {
    {\rm Generate mask} $m_{q,n} = \mathbb{P}(\theta_q, n) $  \\
    {\rm Prune} $\theta_{q,n,0} = \theta_q \odot m_{q,n} $ \\
    {\rm $/ /$ } Update local models: \\
    \For{epoch $t=1,2,\ldots, T$  }
    { {\rm } $\theta_{q,n,t}=\theta_{q,n,t-1}-\gamma \nabla F_n(\theta_{q,n,t-1}, \xi_{n,t-1})\odot m_{q,n} $
    }
}
{\rm $/ /$ } Update global model: \\
\For{region $i=1,2,\ldots, K$ }{
    {\rm Find} $\mathcal{N}_q^{(i)}=\{ n: m_{q,n}^{(i)} = \mathbf{1} \} $ \\
    {\rm Update} $\theta_{q+1}^{(i)} = \frac{1}{|\mathcal{N}_q^{(i)}|} \sum_{n\in \mathcal{N}_q^{(i)} } \theta_{q,n,T}^{(i)}$
}}
Output $\theta_{Q}$
\end{algorithm}



\subsection{Notations and Assumptions}

We make the following assumptions on $F_{1}, \dots F_{n}$. Assumptions~1 is a standard. Assumption~2 follows from \cite{ma2021effective} and implies the noise introduced by pruning is relatively small and bounded. 
Assumptions 3 and 4 are standard for FL convergence analysis following from \cite{zhang2013communication, stich2018local,yu2019parallel, convergenceNoniid} and assume the stochastic gradients to be bounded and unbiased.

\begin{assumption} (Smoothness).
Cost functions $F_{1}, \dots ,F_{N}$ are all L-smooth: $\forall \theta,\phi\in \mathcal{R}^d$ and any $n$, we assume that there exists $L>0$:
\begin{eqnarray}
\| \nabla F_n(\theta) - \nabla F_n(\phi) \| \le L \|\theta-\phi\|.
\end{eqnarray}
\end{assumption}

\begin{assumption}(Pruning-induced Noise). 
We assume that for some $\delta^{2} \in[0,1)$ and any $q,n$, the pruning-induced error is bounded by
\begin{eqnarray}
\left\|\theta_{q}-\theta_{q} \odot m_{q,n}\right\|^{2} \leq \delta^{2}\left\|\theta_{q}\right\|^{2}.
\end{eqnarray}
\end{assumption}

\begin{assumption} (Bounded Gradient).
The expected squared norm of stochastic gradients is bounded uniformly, i.e., for constant $G>0$ and any $n,q,t$:
\begin{eqnarray}
E \left\| \nabla F_{n}(\theta_{q,n,t},\xi_{q,n,t})\right\|^{2} \leq G.
\end{eqnarray}
\end{assumption}

\begin{assumption} (Gradient Noise for IID data).
Under IID data distribution, for any $q,n,t$, we assume that
\begin{eqnarray}
& \mathbb{E}[\nabla F_n(\theta_{q,n,t}, \xi_{n,t})] = \nabla F(\theta_{q,n,t}) \\
& \mathbb{E}\| \nabla F_n(\theta_{q,n,t}, \xi_{n,t}) - \nabla F(\theta_{q,n,t}) \|^2 \le \sigma^2
\end{eqnarray}
for constant $\sigma^2>0$ and independent samples $\xi_{n,t}$.
\end{assumption}

\subsection{Convergence Analysis}

We now analyze heterogeneous FL under arbitrary adaptive online pruning. To the best of our knowledge, this is the first proof that shows general convergence for this family of algorithms to a stationary point of standard FL (in Section~2.1) with smooth cost functions. We will first show convergence for IID data distributions, and by replacing Assumption~4 with a similar Assumption~5, show convergence for non-IID data distributions. We define an important value: 
\begin{eqnarray}
\Gamma_{min}=\min_{q,i} |\mathcal{N}_q^{(i)}|,
\end{eqnarray}
referred to in this paper as the minimum covering index. Since $|\mathcal{N}_q^{(i)}|$ is the number of local models containing parameters of region $i$, $\Gamma_{min}$ measures the minimum occurrence of any parameters in the local models. Intuitively, if a parameter is never included in any local models, it is impossible for it to be updated. Thus conditions based on the covering index would be necessary for the convergence to standard FL. Our analysis establishes sufficient conditions for convergence. All proofs are collected in the Appendix.


\begin{theorem}
Under Assumptions~1-4 and for arbitrary pruning satisfying $\Gamma_{min}\ge 1$, heterogeneous FL with adaptive online pruning converge as follows:
\begin{eqnarray}
\frac{1}{Q} \sum_{q=1}^Q \mathbb{E} ||\nabla F(\theta_q)||^2 \le \frac{G_0}{\sqrt{TQ}} +  \frac{V_0}{Q} + \frac{I_0}{\Gamma^*} \cdot \frac{\delta^2}{Q}\sum_{q=1}^Q \mathbb{E}  \| \theta_q \|^2  \nonumber
\end{eqnarray}
where $V_0=3L^2NG/\Gamma^*$, $I_0=3L^2 N$, and $G_0=4\mathbb{E} [ F(\theta_{0})]+6LN \sigma^2/(\Gamma^*)^2$, are constants depending on the initial model parameters and the gradient noise.
\end{theorem}


\textbf{Remark 2.} Theorem~1 shows convergence to a stationary point of 
standard FL as long as $\Gamma_{min}\ge 1$ (albeit some pruning-induced noise). The result is a bit surprising, since $\Gamma_{min}\ge 1$ only requires any parameters to be included in at least one local model (which is necessary for all parameters to be updated during training). But we show that this is a sufficient condition for convergence to standard FL. Moreover, we also establish a convergence rate of $O(\frac{1}{\sqrt{Q}})$ for arbitrary pruning strategies satisfying the condition.


\textbf{Remark 3.} Impact of pruning-induced noise. In Assumption~2, we assume the pruning-induced noise is relatively small and bounded with respect to the global model: $\left\|\theta_{q}-\theta_{q} \odot m_{q,n}\right\|^{2} \leq \delta^{2}\left\|\theta_{q}\right\|^{2}$. This is satisfied in practice since most pruning strategies tend to focus on eliminating weights/neurons that are insignificant, therefore keeping $\delta^2$ indeed small.
We note that pruning will incur an error term $\delta^2\frac{1}{Q}\sum_{q=1}^Q \mathbb{E}  \| \theta_q \|^2 $ in our convergence analysis, which is proportional to $\delta^2$ and the average model norm (averaged over Q). It implies that more aggressive pruning in heterogeneous FL may lead to a larger error, deviating from standard FL at a speed quantified by $\delta^{2}$. We note that this error is affected by both $\delta^{2}$ and $\Gamma_{min}$.

\textbf{Remark 4.} Impact of minimum covering index $\Gamma_{min}$. It turns out that the minimum number of occurrences of any parameter in the local models is a key factor deciding convergence. As $\Gamma_{min}$ increases, both constants $G_0,V_0$ and the convergence error are inverse proportional to $\Gamma_{min}$. This result is a bit counter-intuitive since certain parameters should be small enough to ignore in pruning. However, recall that our analysis shows convergence of all parameters in $\theta_q$ to a stationary point of standard FL (rather than for a subset of parameters or to a random point). The more times a parameter is covered by local models, the sooner it gets updated and convergences to the desired target. This is quantified in our analysis by showing that the error term due to pruning noise decreases at the rate of $\Gamma_{min}$. 



\textbf{Remark 5.} When the cost function is strongly convex (e.g., for softmax classifier, logistic regression and linear regression with $l$2-normalization), a stationary point becomes the global optimum. Thus, Theorem~1 shows convergence to the global optimum of standard FL for strongly convex cost functions.

\textbf{Remark 6.} Theorem~1 inspires new design of adaptive online pruning for heterogeneous FL. Since the convergence gap is affected by both pruning-induced noise $\delta^2$ and minimum covering index $\Gamma_{min}$, we may want to design pruning masks to preserve the largest parameters while sufficiently covering all parameters in different local models. The example shown in Figure~1 illustrates three alternative pruning strategies for $N=10$ clients. It can be seen that to achieve the best performance in heterogeneous FL, pruning masks need to be optimized to mitigate noise $\delta^2$ and achieve high covering index. Due to space limitations, optimal pruning mask design with respect to clients' resource constraints will be considered in future work. We present numerical examples with different pruning mask designs (with improved performance for low $\delta^2$ and high $\Gamma_{min}$) in Section~5 to support the observation.

\begin{figure}[hb]
\includegraphics[width=0.99\textwidth]{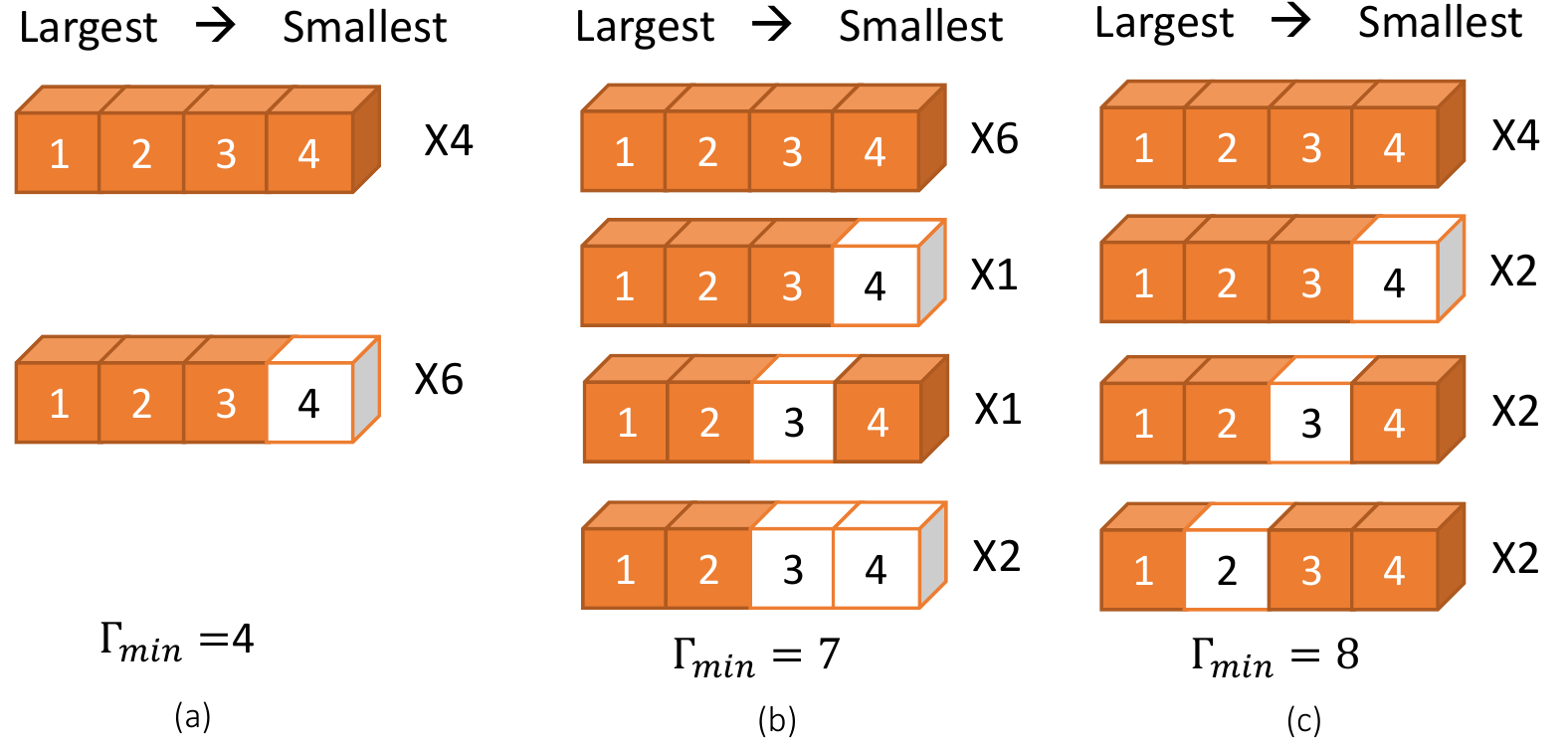}
\caption{Illustration of our method. (a): Existing method utilizing pruning will always discard the parameters below the threshold. (b,c): Our method will utilize different partitions for a higher $\Gamma_{min}$. Note FL with these 3 settings have nearly identical communication and computation costs.}
\end{figure}



When data distribution is non-IID, we need a stronger assumption to ensure stochastic gradients computed on a subset of clients' datasets still provide a non-biased estimate for each parameter region. To this end, we replace Assumption~4 by a similar Assumption~5 for non-IID.

\begin{assumption} (Gradient Noise for non-IID data).
Under non-IID data distribution, we assume that for constant $\sigma^2>0$ and any $q,n,t$:
\begin{eqnarray}
\mathbb{E}[\frac{1}{|\mathcal{N}_q^{(i)}|} \sum_{n\in \mathcal{N}_q^{(i)} } \nabla F_n^{(i)}(\theta_{q,n,t}, \xi_{n,t})] = \nabla F^{(i)}(\theta_{q,n,t}) \nonumber 
\end{eqnarray}
\vspace{-0.3in}
\begin{eqnarray}
\mathbb{E}\left\| \frac{1}{|\mathcal{N}_q^{(i)}|} \sum_{n\in \mathcal{N}_q^{(i)} } \nabla F_n^{(i)}(\theta_{q,n,t}, \xi_{n,t}) -\nabla F^{(i)}(\theta_{q,n,t}) \right\|^2 \le \sigma^2 \nonumber.
\end{eqnarray}
\end{assumption}

\begin{theorem}
Under Assumptions 1-3 and 5, heterogeneous FL with arbitrary adaptive online pruning strategy satisfying $\Gamma_{min}\ge 1$ converges as follows:
\begin{eqnarray}
\frac{1}{Q} \sum_{q=1}^Q \mathbb{E} ||\nabla f(\theta_q)||^2 \le \frac{H_0}{\sqrt{TQ}} +  \frac{U_0}{{Q}} + \frac{\sigma^2I_0}{Q}\sum_{q=1}^T \mathbb{E}  \| \theta_q \|^2 \nonumber
\end{eqnarray}
where $H_0=4\mathbb{E} [ F(\theta_{0})]+6LK\sigma^2$, $U_0=3L^2NG$ and  $I_0$ is same constant as before.
\end{theorem}

\textbf{Remark 7.} With Assumption~5, the convergence under non-IID data distributions is very similar to that in Theorem, except for different constants $H_0=4\mathbb{E} [ F(\theta_{0})]+6LK\sigma^2$ and $U_0=3L^2NG$. Thus, most remarks made for Theorem~1, including convergence speed, pruning-induced noise, and pruning mask design, still apply. We notice that $\Gamma_{min}$ no longer plays a role in the convergence error. This is because the stochastic gradients computed by different clients in $\mathcal{N}_q^{(i)}$ now are based on different datasets and jointly provide an unbiased estimate, no longer resulting in smaller statistical noise.


\textbf{Remark 8.} We note that Assumption~5 can be satisfied in practice by jointly designing pruning masks and data partitions among the clients.
For example, for $N=4$ clients with local data $(D_1, D_2+D_3, D_1+D_2, D_3)$ respectively, we can design 4 pruning masks like $m_{q,1}=[1, 1 ,0]$, $m_{q,2}=[1, 1, 0]$, $m_{q,3}=[0, 1, 1]$, $m_{q,4}=[0, 1, 1]$ or $m_{q,1}=[1, 0 ,1]$, $m_{q,2}=[1, 0, 1]$, $m_{q,3}=[1, 1, 1]$, $m_{q,4}=[1, 1, 1]$. It is easy to show that these satisfy the gradient noise assumption for non-IID data distribution. Due to space limitations, optimal pruning mask design based on data partitioning will be considered in future work. Nevertheless, we present numerical examples under non-IID data distribution with different pruning mask designs in Section~5. When the conditions in Theorem~2 are satisfied, we observe convergence and significant improvement in performance.





\section{Experiments}

\subsection{Experiment settings}

In this section we evaluate different pruning techniques from state-of-the-art designs and verify our proposed theory under unifying pruning framework using two datasets. Unless stated otherwise, the accuracy reported is defined as $\frac{1}{n} \sum _{i} p_{i}\sum_{j}\text{Acc}(f_{i}(x_{j}^{(i)},\theta_{i}\odot m_{i}),y_{j}^{i}))$ averaged over three random seeds with same random initialized starting $\theta_{0}$. 
We focus on three points in our experiments: (i) the general coverage of federated learning with heterogeneous models by pruning (ii) the impact of minimum coverage index $\Gamma_{min}$ (iii) the impact of pruning-induced noise $\delta$.
The experiment results provide a comprehensive comparison among several pruning techniques with modification of our new design to verify the correctness of our theory.

We examine theoretical results on the following two common image classification datasets: MNIST \cite{minist} and CIFAR10 \cite{krizhevsky2009learning}, among $N=100$ workers with IID and non-IID data with participation ratio $c = 0.1$. For IID data, we follow the design of balanced MNIST by \cite{convergenceNoniid}, and similarly obtain balanced CIFAR10. For non-IID data, we obtained balanced partition with label distribution skewed, where the number of the samples on each device is up to at most two out of ten possible classifications. 

\subsection{Baselines and Test Case Notations}
To empirically verify the correctness of our theory, we pick FedAvg, which can be considered federated learning with full local models, and 4 other pruning techniques\footnote[1]{FullNets can be considered as FedAvg \cite{fedavg} without any pruning. We use "WP" for weights pruning as used in PruneFL\cite{jiang2020model}, "NP" for neuron pruning as used in \cite{shao2019privacy}, "FS" for fixed sub-network as used in HeteroFL \cite{diao2021heterofl} and "PT" for pruning with a pre-trained mask as used in \cite{frankle2019lottery}; for notation and demonstration simplicity.} from state-of-the-art federated learning with heterogeneous model designs as baselines. Let $P_{m}=\frac{\|m\|_{0}}{|\theta|}$ be the sparsity of mask $m$, e.g.,$P_{m}=75\%$ for a model when 25 \% of its weights are pruned. Due to page limits, we show selected combinations over 3 pruning levels named \textit{L} (Large), \textit{M} (medium) and \textit{S} (Small): \textit{L}. 60\% workers with full model and 40\% workers with 75\% pruned model; \textit{M}. 40 \% workers with full model and 60\% workers with 75\% pruned model; \textit{S}. 10\% workers with full model, 30\% workers with 75\% pruned model and 60 \% workers with 50\% pruned models. 

For each round 10
devices are randomly selected to run $E$ steps of SGD. We evaluate the averaged model after each global aggregation on the corresponding global objective and show the global loss in Figure 2. We present the key model characteristics as well as their model accuracy after $Q$ rounds of training on MNIST(IID and non-IID) and CIFAR 10 in table 1 and table 2. The FLOPs and Space stand for amortized FLOPs for one local step and memory space needed to store parameters for one model, with their ratio representing the corresponding communication and computation savings compared with FedAvg which uses a full-size model. 

In the experiments, NP, FS, WP, and PT use the same architecture but the latter two are trained with sparse parameters, while FS and NP are trained on actual reduced network size.
To better exemplify and examine the results, we run all experiments on the small model architecture: an MLP with a single hidden layer for MNIST and a LeNet-5 like network with 2 convolutions layers for CIFAR10. As some large DNN models are proved to have the ability to maintain their performance with a reasonable level of pruning, we use smaller networks to avoid the potential influence from very large networks, as well as other tricks and model characteristics of each framework. More details regarding models and experiment design can be found at Appendix.2.

More results including other possible combinations, pruning details with analysis, and other experiment details can be found at Appendix.3.

\begin{figure*}[ht]
     \centering
     \begin{subfigure}[b]{0.49\textwidth}
         \centering
         \includegraphics[width=\textwidth]{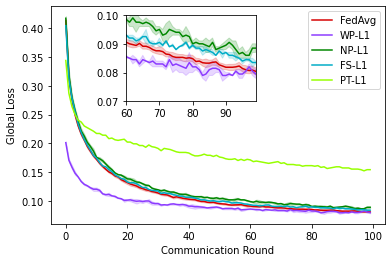}
         \caption{Different Pruning Techniques}
         \label{fig:one}
     \end{subfigure}
     \hfill
     \begin{subfigure}[b]{0.49\textwidth}
         \centering
         \includegraphics[width=\textwidth]{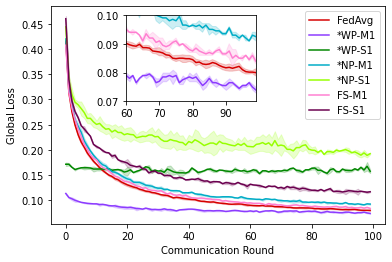}
         \caption{Impact of Pruning Level}
         \label{fig: two}
     \end{subfigure}
     \hfill
     
     \begin{subfigure}[b]{0.49\textwidth}
         \centering
         \includegraphics[width=\textwidth]{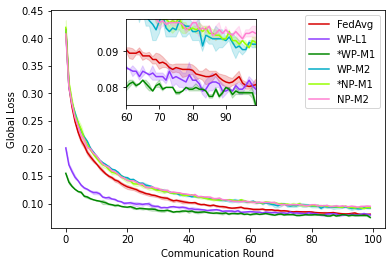}
         \caption{Impact of $\Gamma_{min}$}
         \label{fig:tres}
     \end{subfigure}
        \label{fig:three graphs}
    \begin{subfigure}[b]{0.49\textwidth}
         \centering
         \includegraphics[width=\textwidth]{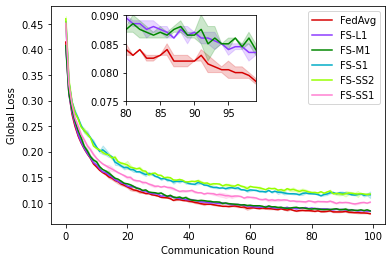}
         \caption{Impact of Mask Error}
         \label{fig:quatro}
     \end{subfigure}
     
    \begin{subfigure}[b]{0.49\textwidth}
         \centering
         \includegraphics[width=\textwidth]{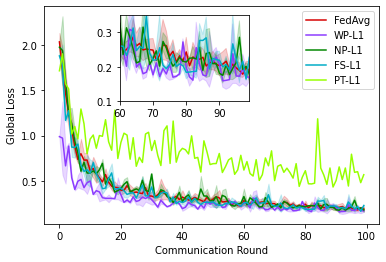}
         \caption{Different Pruning Techniques w/ non-IID}
         \label{fig:cinco}
     \end{subfigure}
    \begin{subfigure}[b]{0.49\textwidth}
         \centering
         \includegraphics[width=\textwidth]{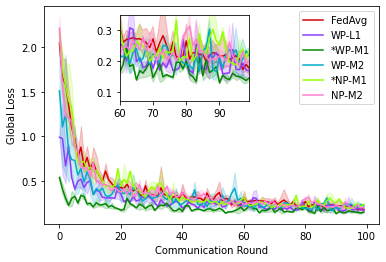}
         \caption{Impact of $\Gamma_{min}$ w/ non-IID}
         \label{fig:sies}
     \end{subfigure}
    \caption{Selected experimental results for MNIST (IID and Non-IID) dataset between different pruning settings. (a) For similar pruning level, PT converge a lot slower than others, as a lottery ticket to an optimal mask  was not found within the limited rounds of training. (b) Generally higher pruning level will lead to a higher loss. (c) By applying our design to increase the coverage index, models with identical architecture can reach a solution with lower loss for both selected pruning techniques without additional computational overhead (d) Relative error introduced by pruning is another key factor to the convergence. (e, f)Similar findings are also observed on such scheme with heterogeneous data.}
\end{figure*}
\clearpage

\begin{table}[h]
\centering
\caption{FL with Different Pruning Techniques on MNIST}
\resizebox{0.99\textwidth}{!}{%
\begin{tabular}{@{}cccccccc@{}}
\toprule
\multirow{3}{*}{Model} & \multirow{3}{*}{FLOPs} & \multirow{3}{*}{Space} & \multirow{3}{*}{Ratio} & \multirow{3}{*}{$\Gamma_{min}$} & \multicolumn{3}{c}{Accuracy} \\ \cmidrule(l){6-8} 
                &         &       &      &    & IID   & \multicolumn{2}{c}{Non-IID} \\ \cmidrule(l){7-8} 
                &         &       &      &    &       & Local        & Global       \\ \midrule
\textit{FullNets} & 158.8K  & 1.27M & 1.00 & 10 & 98.01 & 93.82        & 93.59        \\ \midrule
\textit{WP-L1}  & 143.12K & 1.15M & 0.90 & 6  & 98.18 & 95.49        & 95.15        \\
\textit{NP-L1}  & 142.9K  & 1.14M & 0.90 & 6  & 97.97 & 93.82        & 93.6         \\
\textit{FS-L1}  & 142.9K  & 1.14M & 0.90 & 6  & 97.76 & 92.55        & 92.33        \\ \midrule
\textit{*WP-M1 }  & 135.5K  & 1.08M & 0.85 & 8  & 98.39 & 95.82        & 95.48        \\
\textit{WP-M2}  & 135.5K  & 1.08M & 0.85 & 4  & 97.51 & 89.29        & 89.13        \\
\textit{*NP-M1}  & 135.0K  & 1.08M & 0.85 & 8  & 97.86 & 92.42        & 91.90        \\
\textit{NP-M2}  & 135.0K  & 1.08M & 0.85 & 4  & 97.53 & 92.07        & 91.70        \\
\textit{FS-M1}  & 135.0K  & 1.08M & 0.85 & 4  & 97.62 & 92.33        & 92.05        \\ \midrule
\textit{*WP-S1}  & 100.0K  & 0.80M & 0.63 & 5  & 95.32 & 81.64        & 81.66       \\
\textit{WP-S2}  & 100.0K  & 0.80M & 0.63 & 5  & 95.10& 72.19        & 71.64        \\
\textit{*NP-S1}  & 91.3K   & 0.73M & 0.57 & 3  & 94.41 & 62.49        & 61.96        \\
\textit{NP-S2}  & 91.3K   & 0.73M & 0.57 & 3  & 95.21 & 60.54        & 61.86        \\
\textit{FS-S1}  & 91.3K   & 0.73M & 0.57 & 1  & 96.88 & 90.67        & 90.73        \\ \bottomrule
\end{tabular}%
}
\label{tab:my-table}
\end{table}

\begin{table}[h]
\centering
\caption{FL with Different Pruning Techniques on CIFAR 10 (IID)}
\resizebox{0.995\textwidth}{!}{%
\begin{tabular}{@{}ccccccc@{}}
\toprule
Model & FLOPs & \begin{tabular}[c]{@{}c@{}}FLOPs\\ Ratio\end{tabular} & Space & \begin{tabular}[c]{@{}c@{}}Space\\ Ratio\end{tabular} & $\Gamma_{min}$ & Accuracy \\ \midrule
FullNets & 653.8K  & 1.00  & 512.8K   & 1.00  & 10 & 53.63 \\ \midrule
WP-L1  & 619.6K & 0.94 & 482.3K  & 0.94 & 8  & 53.12 \\
FS-L1  & 619.6K & 0.94 & 476.3K  & 0.93 & 8  & 53.08 \\ \midrule
WP-M1  & 587.0K & 0.89 & 451.9K  & 0.89 & 6  & 52.66 \\
*WP-M2  & 587.0K & 0.89 & 451.9K  & 0.89 & 7  & 52.99 \\
*WP-M3  & 587.0K & 0.89 & 451.9K  & 0.89 & 8  & 54.20 \\
FS-M1  & 585.5K & 0.85 & 440.0K   & 0.89 & 6  & 51.87 \\ \midrule
WP-S1  & 553.7K  & 0.84 & 421.5K   & 0.82 & 4  & 51.69 \\
*WP-S2  & 553.7K  & 0.84 & 421.5K   & 0.82 & 7  & 52.20  \\
FS-S1  & 551.4K & 0.84 & 403.5K & 0.78 & 4  &  50.96     \\ \bottomrule
\end{tabular}%
}

\label{tab:my-table}
\end{table}

\subsection{Impact on minimum coverage index}
Our theory suggests that for a given pruning level, we expect that the minimum coverage index $\Gamma_{min}$ is a hyperbolic function of the global loss as theorem 1 indicates. Then for a given pruning level, a higher minimum coverage index may reduce the the standing point of convergence since $\frac{1}{Q} \sum_{q=1}^Q \mathbb{E} ||\nabla f(\theta_q)||^2$ contains a term $O(\frac{1}{\Gamma_{min}})$, which could potentially lead to better performance. Note that existing pruning techniques and federated learning with heterogeneous models by pruning will always discard the partition in which the parameters are smaller than a certain threshold determined by the pruning policy and pruning level. 

To illustrate the importance of such minimum coverage index, we consider the parameters of a model is argsorted based on certain pruning technique policy $\mathbb{P}$ and then four sets are thereby generated  representing the highest 25\% partition to the lowest 25\% partition: $\mathbb{P}_{1} (\theta)=  \{\textsl{S}_{1},\textsl{S}_{2},\textsl{S}_{3},\textsl{S}_{4}\}$, where a mask generated based on existing pruning technique for a  75\% sparsity model is then defined as $m_{i} = 1 \ \text{if} \ \theta_{i}\in \{\textsl{S}_{1}\cup \textsl{S}_{2}\cup \textsl{S}_{3}\} \ \text{ otherwise} \ m_{i} = 0$. It is easy to see $\Gamma_{min}$ is then directly determined by the number of models with highest pruning levels, e.g. $\Gamma_{min} = 4$  for experiment case 2: 40 \% workers with full models and 60\% workers with 75\% pruned models.

To verify the impact of the minimum coverage index, we propose a way to increase the minimum coverage index without changing the network architecture or introducing extra computation overhead: by a simple way to increase the usage of parameters below pruning threshold: let partial local models train with pruned partition.

As an example shown in Fig 1, for model with code name *WP-M1, from which 2 out of 6 models with 75\% pruned model using regular weights pruning technique, with the other 4 each two use
$\mathbb{P}_{2} (\theta)=  \{\textsl{S}_{1},\textsl{S}_{3},\textsl{S}_{4}\}$ with 
$m_{i} = 1 \ \text{if} \ \theta_{i}\in \{\textsl{S}_{1}\cup \textsl{S}_{3}\cup \textsl{S}_{4}\} \ \text{ otherwise} \ m_{i} = 0$
and $\mathbb{P}_{3} (\theta)=  \{\textsl{S}_{1},\textsl{S}_{2},\textsl{S}_{4}\}$ with 
$m_{i} = 1 \ \text{if} \ \theta_{i}\in \{\textsl{S}_{1}\cup \textsl{S}_{2}\cup \textsl{S}_{4}\} \ \text{ otherwise} \ m_{i} = 0$
, so that $\Gamma_{min} = 8$ is then achieved. We denote such design to maximize the minimum coverage index on current pruning techniques as STAR(*) in the results.
For detailed case settings and corresponding pruning techniques see Appendix 2. 

As shown in Figure 2(c), under the identical model setting with the same pruning level, both pruning techniques with different minimum coverage index show different convergence scenarios, specifically, the design with higher minimum coverage index is able to reach a solution with lower loss within the training round limits. It is also observed in table 1 and table 2 that they can reach higher accuracy with both IID and non-IID data, whereas for non-IID the improvements are more significant. 
There are even cases where settings under our design with fewer communication and computation costs that perform better than a regular design with more costs, e.g. "*WP-M1" over "WP-L1" on both IID and non-IID data. More examples and results can be found in Appendix.3

\subsection{Impact on pruning-induced noise}
As suggested in our theory, another key factor that contributes to the convergence is the pruning-induced noise $\delta$. When the client model is treated with pruning or a sparse mask, then inevitably the pruning-induced noise $\delta$ will affect the convergency and model accuracy. Given the same minimum coverage index, generally, a smaller pruning-induced noise will lead to a lower convergence point and potentially a more accurate model. 


For this phenomenon, we focus on the Fixed Sub-network method, which does not involve the adaptive changing of mask, and tested higher pruning levels as shown in figure 2(d) and confirms such a trend. As shown in Figure 2(b), all selected pruning methods are affected by the change of pruning level. In Figure 2(f), where model WP2-1 trains with a relatively steady trend, it becomes unsteady for model WP3-1, which could be due to the change of pruning mask before its local convergence. This may suggest that on high pruning level without a properly designed pruning technique, using fixed sub-network may bring a more robust and accurate model.

\subsection{More Discussions and Empirical Findings}
Besides on the verification of our theory, we have additionally noticed several phenomenon, of which some confirm previous research while some others may require further investigation for a theoretical support. 

In Figure 2(a), under a similar pruning level, PT converges a lot slower than others with its pre-trained mask, as also suggested by previous works that models with static sparse parameters will converge to a solution with higher loss than models with dynamic sparse training. Also, it suggests that it is unlikely to find such a lottery ticket to an optimal mask within limited rounds of training, especially without a carefully designed algorithm. 

Although generally a higher pruning level will result in a higher loss in training, different pruning techniques have different sensitivity towards it. As an example  Fixed Sub-network has a relatively low sensitivity towards high pruning levels, this could be due to by using the static continuous mask will avoid the situation where pruning mask is changed before its local convergence, which makes it more stable on different pruning levels. 

Nevertheless, for most cases our design to increase minimum coverage index could deliver improvement on increasing model accuracy and reducing the global loss, pruning-induced noise is another key factor to notice, especially with higher pruning levels such design to merely focus on increasing minimum coverage index may not bring significant improvements.

Finally, we also show a synthetic special case where (the proposed necessary conditions are not met)all local clients do not sum up a mask that covers the whole model in Appendix.4, which under this situation the model did not learn a usable solution.

\section{Conclusion}
In this paper, we establish (for the first time) the sufficient conditions for FL with heterogeneous local models and arbitrary adaptive pruning to converge to a stationary point of standard FL, at a rate of $\frac{1}{\sqrt{Q}}$. It applies to general smooth cost functions and recovers a number of important FL algorithms as special cases. The analysis advocates designing pruning strategies with respect to both minimum coverage index $\Gamma_{\rm min}$ and pruning-induced noise $\delta^2$.
We further empirically demonstrated the correctness of the theory and the performances of the proposed design. 
Our work provides a theoretical understanding of FL with heterogeneous clients and dynamic pruning, and presents valuable insights on FL algorithm design, which will be considered in future work.

\bibliographystyle{plain} 
\bibliography{main} 

\begin{appendices}

\section{Proof of Theorems 1 and 2}

\subsection{Problem summary and notations} 

We summarize the algorithm in a way that can present the convergence analysis more easily. We use a superscript such as $\theta^{(i)}$, $m_{q,n}^{(i)}$, and $\nabla{F}^{(i)}$ to denote the sub-vector of parameter, mask, and gradient corresponding to region $i$. In each round $q$, parameters in each region $i$ is contained in and only in a set of local models denoted by $\mathcal{N}_q^{(i)}$, implying that $m_{q,n}^{(i)} = \mathbf{1}$ for $n\in \mathcal{N}_q^{(i)}$ and $m_{q,n}^{(i)} = \mathbf{0}$ otherwise. We define $\Gamma^*=\min_{q,i} \mathcal{N}_q^{(i)}$ as the minimum coverage index, since it denotes the minimum number of local models that contain any parameters in $\theta_q$. With slight abuse of notations, we use $\nabla F_n(\theta$ and $\nabla F_n(\theta,\xi)$ to denote the the gradient and stochastic gradient, respectively. 

\begin{algorithm}[h]
\SetKwInput{KwData}{Input}
\caption{Heterogeneous FL with adaptive online model pruning}\label{alg:cap}
\KwData {Local data ${D_{i}^{k}}$ on $N$ local workers, learning rate $\gamma$, pruning policy $\mathbb{P}$, number of local epochs $T$, global model parameterized by $\theta$.
}
\SetKwFunction{FMain}{Local Update}
  \SetKwProg{Fn}{Function}{:}{}
\textbf{Executes:} \\
Initialize $\theta_0$\\
\For{round $q=1,2,\ldots, Q$ {\bf do}}{
\For{local workers $n=1,2,\ldots, N$ {\bf do} (In parallel)} {
    {\rm Generate mask} $m_{q,n} = \mathbb{P}(\theta_q, n) $  \\
    {\rm Prune} $\theta_{q,n,0} = \theta_q \odot m_{q,n} $ \\
    {\rm $/ /$ } Update local models: \\
    \For{epoch $t=1,2,\ldots, T$ {\bf do} }
    { {\rm Update} $\theta_{q,n,t}=\theta_{q,n,t-1}-\gamma \nabla F_n(\theta_{q,n,t-1}, \xi_{n,t-1})\odot m_{q,n} $
    }
}
{\rm $/ /$ } Update global model: \\
\For{region $i=1,2,\ldots, K$ {\bf do} }{
    {\rm Find} $\mathcal{N}_q^{(i)}=\{ n: m_{q,n}^{(i)} = \mathbf{1} \} $ \\
    {\rm Update} $\theta_{q+1}^{(i)} = \frac{1}{|\mathcal{N}_q^{(i)}|} \sum_{n\in \mathcal{N}_q^{(i)} } \theta_{q,n,T}^{(i)}$
}}
Output $\theta_{Q}$
\end{algorithm}

\subsection{Assumptions}

\begin{assumption} (Smoothness).
Cost functions $F_{1}, \dots ,F_{N}$ are all L-smooth: $\forall \theta,\phi\in \mathcal{R}^d$ and any $n$, we assume that there exists $L>0$:
\begin{eqnarray}
\| \nabla F_n(\theta) - \nabla F_n(\phi) \| \le L \|\theta-\phi\|.
\end{eqnarray}
\end{assumption}

\begin{assumption}(Pruning-induced Error). 
We assume that for some $\delta^{2} \in[0,1)$ and any $q,n,t$, the pruning-induced error is bounded by
\begin{eqnarray}
\left\|\theta_{q,n,t}-\theta_{q,n,t} \odot m_{q,n}\right\|^{2} \leq \delta^{2}\left\|\theta_{q,n,t}\right\|^{2}.
\end{eqnarray}
\end{assumption}

\begin{assumption} (Bounded Gradient).
The expected squared norm of stochastic gradients is bounded uniformly, i.e., for constant $G>0$ and any $n,q,t$:
\begin{eqnarray}
E \left\| \nabla F_{n}(\theta_{q,n,t},x_{q,n,t})\right\|^{2} \leq G.
\end{eqnarray}
\end{assumption}

\begin{assumption} (Gradient Noise for IID data).
Under IID data distribution, for any $q,n,t$, we assume that
\begin{eqnarray}
& \mathbb{E}[\nabla F_n(\theta_{q,n,t}, \xi_{n,t})] = \nabla F(\theta_{q,n,t}) \\
& \mathbb{E}\| \nabla F_n(\theta_{q,n,t}, \xi_{n,t}) - \nabla F(\theta_{q,n,t}) \|^2 \le \sigma^2
\end{eqnarray}
where $\sigma^2>0$ is a constant and $\xi_{n,t})$ are independent samples for different $n,t$.
\end{assumption}

\begin{assumption} (Gradient Noise for non-IID data).
Under non-IID data distribution, we assume that for constant $\sigma^2>0$ and any $q,n,t$:
\begin{eqnarray}
& \mathbb{E}\left[\frac{1}{|\mathcal{N}_q^{(i)}|} \sum_{n\in \mathcal{N}_q^{(i)} } \nabla F_n^{(i)}(\theta_{q,n,t}, \xi_{n,t})\right] = \nabla F^{(i)}(\theta_{q,n,t}) \\
& \mathbb{E}\left\| \frac{1}{|\mathcal{N}_q^{(i)}|} \sum_{n\in \mathcal{N}_q^{(i)} } \nabla F_n^{(i)}(\theta_{q,n,t}, \xi_{n,t}) -\nabla F^{(i)}(\theta_{q,n,t}) \right\|^2 \le \sigma^2.
\end{eqnarray}
\end{assumption}

\subsection{Convergence Analysis} 


We now analyze the convergence of heterogeneous FL under adaptive online model pruning with respect to any pruning policy $\mathbb{P}(\theta_q,n)$ (and the resulting mask $m_{q,n}$) and prove the main theorems in this paper. We need to overcome a number of challenges as follows:
\begin{itemize}
\vspace{-0.05in}
    \item We will begin the proof by analyzing the change of loss function in one round as the model goes from $\theta_q$ to $\theta_{q+1}$, i.e., $F(\theta_{q+1})-F(\theta_1)$ . It includes three major steps: pruning to obtain heterogeneous local models $\theta_{q,n,0}=\theta_q\odot{m_{q,n}}$, training local models in a distributed fashion to update $\theta_{q,n,t}$, and parameter aggregation to update the global model $\theta_{q+1}$.
    \vspace{-0.05in}
    \item Due to the use of heterogeneous local models whose masks $m_{q,n}$ both vary over rounds and change for different workers, we first characterize the difference between (i) local model $\theta_{q,n,t}$ at any epoch $t$ and global model $\theta_{q}$ at the beginning of the current round. It is easy to see that this can be factorized into two parts: pruning induced error $\|\theta_{q,n,0}- \theta_{q}\|^2$ and local training $\|\theta_{q,n,t}- \theta_{q,n,0}\|^2$, which will be analyzed in Lemma~1.
    \vspace{-0.05in}
    \item We characterize the impact of heterogeneous local models on global parameter update. Specifically, We use an ideal local gradient $\nabla F_n(\theta_q)$ as a reference point and quantify the different between aggregated local gradients and the ideal gradient. This will be presented in Lemma~2. We also quantify the norm difference between a gradient and a stochastic gradient (with respect to the global update step) using the gradient noise assumptions, in Lemma~3.
    \vspace{-0.05in}
    \item Since IID and non-IID data distributions in our model differ in the gradient noise assumption (i.e., Assumption~4 and Assumption~5), we present a unified proof for both cases. We will explicitly state IID and non-IID data distributions only if the two cases require different treatment (when the gradient noise assumptions are needed). Otherwise, the derivations and proofs are identical for both cases.
\end{itemize}

We will begin by proving a number of lemmas and then use them for convergence analysis.

\begin{lemma}
Under Assumption~2 and Assumption~3, for any $q$, we have:
\begin{eqnarray}
\sum_{t=1}^T \sum_{n=1}^N \mathbb{E} \| \theta_{q,n,t-1} - \theta_{q} \|^2 \le \gamma^2 T^2NG+ \delta^2  NT \cdot \mathbb{E} \|\theta_{q} \|^2 .
\end{eqnarray}
\end{lemma}
\begin{proof}
We note that $\theta_{q}$ is the global model at the beginning of current round. We split the difference $\theta_{q,n,t-1} - \theta_{q}$ into two parts: changes due to local model training $\theta_{q,n,t-1} - \theta_{q,n,0}$ and changes due to pruning $\theta_{q,n,0} - \theta_{q}$. That is
\begin{eqnarray}
& & \sum_{t=1}^T \sum_{n=1}^N \mathbb{E} \| \theta_{q,n,t-1} - \theta_{q} \|^2 \nonumber \\ 
& & \ \ \ \ \ = \sum_{t=1}^T \sum_{n=1}^N \mathbb{E} \| \left( \theta_{q,n,t-1} - \theta_{q,n,0} \right) + \left( \theta_{q,n,0} - \theta_{q} \right)  \|^2 \nonumber \\
& & \ \ \ \ \ \le \sum_{t=1}^T \sum_{n=1}^N 2\mathbb{E} \| \theta_{q,n,t-1} - \theta_{q} \|^2 + \sum_{t=1}^T \sum_{n=1}^N 2\mathbb{E} \| \theta_{q,n,t-1} - \theta_{q} \|^2  \label{l1_1}
\end{eqnarray}
where we used the fact that $\|\sum_{i=1}^s a_i \|^2\le s \sum_{i=1}^s \|a_i \|^2 $ in the last step. 

For the first term in Eq.(\ref{l1_1}), we notice that $\theta_{q,n,t-1}$ is obtained from $\theta_{q,n,0}$ through $t-1$ epochs of local model updates on worker $n$. Using the local gradient updates from the algorithm, it is easy to see:
\begin{eqnarray}
& & \sum_{t=1^T} \sum_{n=1}^N \mathbb{E} \| \theta_{q,n,t-1} - \theta_{q,n,0} \|^2 \nonumber \\ 
& & \ \ \ \ \ = \sum_{t=1^T} \sum_{n=1}^N  \mathbb{E}\left\| \sum_{j=0}^{t-1} -\gamma \nabla F_n(\theta_{q,n,t-1}, \xi_{n,t-1})\odot m_{q,n} \right\|^2 \nonumber \\
& & \ \ \ \ \  \le  \sum_{t=1^T} \sum_{n=1}^N (t-1)  \sum_{j=0}^{t-1} \mathbb{E}\left\| -\gamma \nabla F_n(\theta_{q,n,t-1}, \xi_{n,t-1})\odot m_{q,n} \right\|^2 \nonumber \\
& & \ \ \ \ \  \le \sum_{t=1^T} \sum_{n=1}^N (t-1) \gamma^2 G \nonumber \\
& & \ \ \ \ \  \le \frac{\gamma^2 T^2 NG}{2}, \label{l1_2}
\end{eqnarray}
where we use the fact that $\|\sum_{i=1}^s a_i \|^2\le s \sum_{i=1}^s \|a_i \|^2 $ in step 2 above, and the fact that $m_{q,n}$ is a binary mask in step 3 above together with Assumption~3 for bounded gradient.  

For the second term in Eq.(\ref{l1_1}), the difference is resulted by model pruning using mask $m_{n,q}$ of work $n$ in round $q$. We have
\begin{eqnarray}
& \sum_{t=1}^T \sum_{n=1}^N \mathbb{E} \| \theta_{q,n,0} - \theta_{q} \|^2 &  = \sum_{t=1}^T \sum_{n=1}^N \mathbb{E} \| \theta_{q}\odot m_{n,q} - \theta_{q} \|^2 \nonumber \\ 
& & \le \sum_{t=1}^T \sum_{n=1}^N \delta^2  \mathbb{E} \|\theta_{q} \|^2 \nonumber \\
& & = \delta^2  NT \cdot \mathbb{E} \|\theta_{q} \|^2, \label{l1_3}
\end{eqnarray}
where we used the fact that $\theta_{q,n,0}=\theta_{q}\odot m_{n,q}$ in step 1 above, and Assumption~2 in step 2 above.

Plugging Eq.(\ref{l1_2}) and Eq.(\ref{l1_3}) into Eq.(\ref{l1_1}), we obtain the desired result.
\end{proof}

\begin{lemma}
Under Assumptions~1-3, for any $q$, we have:
\begin{eqnarray}
& & \sum_{i=1}^K \mathbb{E} \left\| \frac{1}{\Gamma_q^{(i)}T} \sum_{t=1}^T \sum_{n\in\mathcal{N}_q^{(i)}} \left[ \nabla F_n^{(i)}({\theta_{q,n,t-1}})-\nabla F_n^{(i)} (\theta_q)  \right] \right\|^2  \nonumber \\ 
& &  \ \ \ \ \ \ \ \ \ \ \ \ \ \ \ \ \  \ \ \ \ \ \ \ \ \ \ \ \ \ \ \ \ \ \ \ \ \ \ \le  \frac{L^2 \gamma^2 T NG}{\Gamma^*} + \frac{L^2\delta^2  N}{\Gamma^*}  \mathbb{E} \|\theta_{q} \|^2
\label{l2_0}.
\end{eqnarray}
\end{lemma}
\begin{proof}
Recall that $\Gamma_q^{(i)}=|\mathcal{N}_q^{(i)}|$ is the number of local models containing parameters of region $i$ in round $q$. The left-hand-side of Eq.(\ref{l2_0}) denotes the difference between an average gradient of heterogeneous models (through aggregation and over time) and an ideal gradient. The summation over $i$ adds up such difference over all regions $i=1,\ldots,K$, because the average gradient takes a different form in different regions.

From the inequality $\|\sum_{i=1}^s a_i \|^2\le s \sum_{i=1}^s \|a_i \|^2 $, we obtain $\| \frac{1}{s} \sum_{i=1}^s  a_i \|^2\le \frac{1}{s} \sum_{i=1}^s \|a_i \|^2 $. We use this inequality on the left-hand-side of Eq.(\ref{l2_0}) to get:
\begin{eqnarray}
& & \sum_{i=1}^K \mathbb{E} \left\| \frac{1}{\Gamma_q^{(i)}T} \sum_{t=1}^T \sum_{n\in\mathcal{N}_q^{(i)}} \left[ \nabla F_n^{(i)}({\theta_{q,n,t-1}})-\nabla F_n^{(i)} (\theta_q)  \right] \right\|^2 \nonumber \\
& & \ \ \ \ \ \le \sum_{i=1}^K \frac{1}{\Gamma_q^{(i)}T} \sum_{t=1}^T \sum_{n\in\mathcal{N}_q^{(i)}} \mathbb{E} \left\|  \nabla F_n^{(i)}({\theta_{q,n,t-1}})-\nabla F_n^{(i)} (\theta_q)  \right\|^2 \nonumber \\
& &  \ \ \ \ \ \le \frac{1}{T\Gamma^{*}}  \sum_{t=1}^T \sum_{n=1}^N \sum_{i=1}^K \mathbb{E} \left\|  \nabla F_n^{(i)}({\theta_{q,n,t-1}})-\nabla F_n^{(i)} (\theta_q)  \right\|^2 \nonumber \\
& &  \ \ \ \ \ = \frac{1}{T\Gamma^{*}}  \sum_{t=1}^T \sum_{n=1}^N \mathbb{E} \left\|  \nabla F_n({\theta_{q,n,t-1}})-\nabla F_n (\theta_q)  \right\|^2 \nonumber \\
& & \ \ \ \ \ \le \frac{1}{T\Gamma^{*}}  \sum_{t=1}^T \sum_{n=1}^N L^2 \mathbb{E} \left\|  {\theta_{q,n,t-1}} - \theta_q  \right\|^2 {\label{l2_2}},
\end{eqnarray}
where we relax the inequality by choosing the smallest $\Gamma^*=\min_{q,i} \Gamma_q^{(i)}$ and changing the summation over $n$ to all workers in the second step. In the third step, we use the fact that $L_2$ gradient norm of a vector is equal to the sum of norm of all sub-vectors (i.e., regions $i=1,\ldots,K$). This allows us to consider $\nabla F_n$ instead of its sub-vectors on different regions. 

Finally, the last step is directly from L-smoothness in Assumption~1. Under Assumptions~2-3, we notice that the last step of Eq.(\ref{l2_2}) is further bounded by Lemma~1, which yields the desired result of this lemma after re-arranging the terms.
\end{proof}

\begin{lemma}
For IID data distribution under Assumptions~4, for any $q$, we have:
\begin{eqnarray}
\sum_{i=1}^K \mathbb{E} \left\| \frac{1}{\Gamma_q^{(i)}T} \sum_{t=1}^T \sum_{n\in\mathcal{N}_q^{(i)}} \left[ \nabla F_n^{(i)}({\theta_{q,n,t-1}},\xi_{n,t-1})-\nabla F^{(i)} (\theta_{q,n,t-1})  \right] \right\|^2   \le  \frac{N\sigma^2}{T({\Gamma^*})^2}
\nonumber \label{l3_0}.
\end{eqnarray}
For non-IID data distribution under Assumption~5, for any $q$, we have:
\begin{eqnarray}
\sum_{i=1}^K \mathbb{E} \left\| \frac{1}{\Gamma_q^{(i)}T} \sum_{t=1}^T \sum_{n\in\mathcal{N}_q^{(i)}} \left[ \nabla F_n^{(i)}({\theta_{q,n,t-1}},\xi_{n,t-1})-\nabla F^{(i)} (\theta_{q,n,t-1})  \right] \right\|^2   \le  \frac{K\sigma^2}{T}
\nonumber \label{l3_0}.
\end{eqnarray}
\end{lemma}
\begin{proof}
This lemma quantifies the square norm of the difference between gradient and stochastic gradient in the global parameter update. We present results for both IID and non-IID cases in this lemma under Assumption~4 and Assumption~5, respectively. 

We first consider IID data distributions. Since all the samples $\xi_{n,t-1}$ are independent from each other for different $n$ and $t-1$, the difference between gradient and stochastic gradient, i.e., $\nabla F_n^{(i)}({\theta_{q,n,t-1}},\xi_{n,t-1})-\nabla F_n^{(i)} (\theta_{q,n,t-1})$, are independent gradient noise. Due to Assumption~4, these gradient noise has zero mean. Using the fact that $\mathbb{E}\|\sum_i \mathbf{x}_i\|^2 = \sum_i \mathbb{E} \|\mathbf{x}_i^2\| $ for zero-mean and independent $\mathbf{x}_i$'s, we get:
\begin{eqnarray}
& & \sum_{i=1}^K \mathbb{E} \left\| \frac{1}{\Gamma_q^{(i)}T} \sum_{t=1}^T \sum_{n\in\mathcal{N}_q^{(i)}} \left[ \nabla F_n^{(i)}({\theta_{q,n,t-1}},\xi_{n,t-1})-\nabla F_n^{(i)} (\theta_{q,n,t-1})  \right] \right\|^2  \nonumber \\ 
& &  \ \ \ \ \ \le \sum_{i=1}^K \frac{1}{(\Gamma_q^{(i)}T)^2} \sum_{t=1}^T \sum_{n\in\mathcal{N}_q^{(i)}} \mathbb{E} \left\| \nabla F_n^{(i)}({\theta_{q,n,t-1}},\xi_{n,t-1})-\nabla F_n^{(i)} (\theta_{q,n,t-1})  \right\|^2  \nonumber \\
& & \ \ \ \ \ \le \frac{1}{(T\Gamma^*)^2}\sum_{i=1}^K \sum_{t=1}^T \sum_{n=1}^N \mathbb{E} \left\| \nabla F_n^{(i)}({\theta_{q,n,t-1}},\xi_{n,t-1})-\nabla F_n^{(i)} (\theta_{q,n,t-1})  \right\|^2  \nonumber \\
& & \ \ \ \ \ = \frac{1}{(T\Gamma^*)^2} \sum_{t=1}^T \sum_{n=1}^N \mathbb{E} \left\| \nabla F_n({\theta_{q,n,t-1}},\xi_{n,t-1})-\nabla F_n (\theta_{q,n,t-1})  \right\|^2 \nonumber \\
& & \ \ \ \ \ \le \frac{1}{(T\Gamma^*)^2} \cdot TN\sigma^2
\end{eqnarray}
where we used the property of zero-mean and independent gradient noise in the first step above, relax the inequality by choosing the smallest $\Gamma^*=\min_{q,i} \Gamma_q^{(i)}$ and changing the summation over $n$ to all workers in the second step. In the third step, we use the fact that $L_2$ gradient norm of a vector is equal to the sum of norm of all sub-vectors (i.e., regions $i=1,\ldots,K$). This allows us to consider $\nabla F_n$ instead of its sub-vectors on different regions. Finally, we apply Assumption~4 to bound the gradient noise and obtain the desired result.

For non-IID data distributions under Assumption~4 (instead of Assumption~5), we notice that
$\mathbb{E}\left[\frac{1}{|\mathcal{N}_q^{(i)}|} \sum_{n\in \mathcal{N}_q^{(i)} } \nabla F_n^{(i)}(\theta_{q,n,t-1}, \xi_{n,t-1})\right] = \nabla F^{(i)}(\theta_{q,n,t-1})$ is an unbiased estimate for any epoch $t$, with bounded gradient noise. Again, due to independent samples $\xi_{n,t-1}$, we have:
\begin{eqnarray}
& & \sum_{i=1}^K \mathbb{E} \left\| \frac{1}{\Gamma_q^{(i)}T} \sum_{t=1}^T \sum_{n\in\mathcal{N}_q^{(i)}} \left[ \nabla F_n^{(i)}({\theta_{q,n,t-1}},\xi_{n,t-1})-\nabla F_n^{(i)} (\theta_{q,n,t-1})  \right] \right\|^2  \nonumber \\ 
& &  \ \ \ \ \ \le \frac{1}{T^2} \sum_{i=1}^K  \sum_{t=1}^T  \mathbb{E} \left\| \frac{1}{\Gamma_q^{(i)}} \sum_{n\in\mathcal{N}_q^{(i)}} \nabla F_n^{(i)}({\theta_{q,n,t-1}},\xi_{n,t-1})-\nabla F_n^{(i)} (\theta_{q,n,t-1})  \right\|^2  \nonumber \\
& & \ \ \ \ \ \le \frac{1}{T^2} \sum_{i=1}^K \sum_{t=1}^T \sigma^2  \nonumber \\
& & \ \ \ \ \ = \frac{K\sigma^2}{T},
\end{eqnarray}
where we use the property of zero-mean and independent gradient noise in the first step above, used the fact that the norm of a sub-vector (in region $i$) is bounded by that of the entire vector in the second step above, as well as Assumption~5. This completes the proof of this lemma.
\end{proof}

\vspace{0.07in}
\noindent {\bf Proof of the main result}. Now we are ready to present the main proof. We begin with the L-smoothness property in Assumption~1, which implies 
\begin{eqnarray}
F(\theta_{q+1}) - F(\theta_q) \le \left< \nabla F(\theta_q) , \  \theta_{q+1} - \theta_{q} \right> + \frac{L}{2 }\left\| \theta_{q+1} - \theta_{q} \right\|^2.
\end{eqnarray}
We take expectations on both sides of the inequality and get:
\begin{eqnarray}
\mathbb{E} [ F(\theta_{q+1})] - \mathbb{E}[]F(\theta_q)] \le \mathbb{E}\left< \nabla F(\theta_q) , \ \theta_{q+1} - \theta_{q} \right> + \frac{L}{2 } \mathbb{E}\left\| \theta_{q+1} - \theta_{q} \right\|^2. \label{mm_0}
\end{eqnarray}
In the following, we bound the two terms on the right-hand-side above and finally combine the results to complete the proof.

\vspace{0.07in}
\noindent {\bf Upperbound for $\mathbb{E}\left< \nabla F(\theta_q) , \ \theta_{q+1} - \theta_{q} \right>$}. We notice that the inner product can be broken down and reformulated as the sum of inner products over all regions $i=1,\ldots,K$. This is necessary because the global parameter update is different for different regions. More precisely, for any region $i$, we have:
\begin{eqnarray}
& \theta_{q+1}^{(i)} - \theta_{q}^{(i)} & =  \left(\frac{1}{\Gamma_q^{(i)}} \sum_{n\in\mathcal{N}_q^{(i)}} \theta_{q,n,T}^{(i)} \right)- \theta_{q}^{(i)} \nonumber \\
& & = \frac{1}{\Gamma_q^{(i)}} \sum_{n\in\mathcal{N}_q^{(i)}}  \left[ \theta_{q,n,0}^{(i)} - \sum_{t=1}^T \gamma \nabla F_n^{(i)} (\theta_{q,n,t-1},\xi_{n,t-1})\cdot m_{n,q}^{(i)} \right] - \theta_{q}^{(i)} \nonumber \\
& &  = - \frac{1}{\Gamma_q^{(i)}} \sum_{n\in\mathcal{N}_q^{(i)}} \sum_{t=1}^T \gamma \nabla F_n^{(i)} (\theta_{q,n,t-1},\xi_{n,t-1})\cdot m_{n,q}^{(i)} + \theta_{q}^{(i)}\cdot m_{n,q}^{(i)} - \theta_{q}^{(i)} \nonumber \\
& & = - \frac{1}{\Gamma_q^{(i)}} \sum_{n\in\mathcal{N}_q^{(i)}} \sum_{t=1}^T \gamma \nabla F_n^{(i)} (\theta_{q,n,t-1},\xi_{n,t-1}), \label{m_1}
\end{eqnarray}
where global parameter updated is used in the first step, local parameter update is used in the second step, and the third step follows from the fact that for any worker $n\in\mathcal{N}_q^{(i)}$ participating in the global update of $\theta^{(i)}_q$ contain the model parameters of region $i$, i.e., $m_{q,n}^{(i)}={\bf 1}$. We also use $ \theta_{q,n,0}^{(i)}=\theta_{q}^{(i)}\cdot m_{n,q}^{(i)}$ in the third step above because of to pruning.

Next we analyze $\mathbb{E}\left< \nabla F(\theta_q) , \ \theta_{q+1} - \theta_{q} \right>$ by considering a sum of inner products over $K$ regions. We have
\begin{eqnarray}
& & \mathbb{E}\left< \nabla F(\theta_q) , \ \theta_{q+1} - \theta_{q} \right> \nonumber \\
& & \ \ \ \ \ \ \ \ \ \ = \sum_{i=1}^K \mathbb{E}\left< \nabla F^{(i)}(\theta_q) , \ \theta_{q+1}^{(i)} - \theta_{q}^{(i)} \right> \nonumber \\
& & \ \ \ \ \ \ \ \ \ \ = \sum_{i=1}^K \mathbb{E}\left< \nabla F^{(i)}(\theta_q) , \ - \frac{1}{\Gamma_q^{(i)}} \sum_{n\in\mathcal{N}_q^{(i)}} \sum_{t=1}^T \gamma \nabla F_n^{(i)} (\theta_{q,n,t-1},\xi_{n,t-1}) \right> \nonumber \\
& & \ \ \ \ \ \ \ \ \ \ = \sum_{i=1}^K \mathbb{E}\left< \nabla F^{(i)}(\theta_q) , \ - \frac{1}{\Gamma_q^{(i)}} \sum_{n\in\mathcal{N}_q^{(i)}} \sum_{t=1}^T \gamma \mathbb{E} \left[ \nabla F_n^{(i)} (\theta_{q,n,t-1},\xi_{n,t-1}) | \theta_q \right]\right> \nonumber \\
& &  \ \ \ \ \ \ \ \ \ \ = \sum_{i=1}^K \mathbb{E}\left< \nabla F^{(i)}(\theta_q) , \ - \frac{1}{\Gamma_q^{(i)}} \sum_{n\in\mathcal{N}_q^{(i)}} \sum_{t=1}^T \gamma  \nabla F_n^{(i)} (\theta_{q,n,t-1})  \right> \nonumber \\
& & \ \ \ \ \ \ \ \ \ \ = - \sum_{i=1}^K \mathbb{E}\left< \nabla F^{(i)}(\theta_q) , \   \gamma T\nabla F^{(i)}(\theta_q) \right> \label{m_2} \\
& & \ \ \ \ \ \ \ \ \ \  \ \ \  -\sum_{i=1}^K \mathbb{E}\left< \nabla F^{(i)}(\theta_q) , \   \frac{1}{\Gamma_q^{(i)}} \sum_{n\in\mathcal{N}_q^{(i)}} \sum_{t=1}^T \gamma \left[ \nabla F_n^{(i)} (\theta_{q,n,t-1}) - \nabla F^{(i)}(\theta_q)\right] \right> \nonumber
\end{eqnarray}
where we use the first step reformulates the inner product as a sum, the second step follows from Eq.(\ref{m_1}), the third step employs a conditional expectation over the random samples with respect to $\theta_q$, and the last step splits the result into two parts with respect to a reference point $\gamma T \nabla F^{(i)}(\theta_q)$. 

For the first term on the right-hand-side of Eq.(\ref{m_2}), it is easy to see that 
\begin{eqnarray}
&  - \sum_{i=1}^K \mathbb{E}\left< \nabla F^{(i)}(\theta_q) , \   \gamma T \nabla F^{(i)}(\theta_q) \right> & = - \gamma T \sum_{i=1}^K \left\| \nabla F^{(i)}(\theta_q)  \right\|^2  \nonumber \\ 
& & = - \gamma T \left\| \nabla F(\theta_q)  \right\|^2, \label{m_3}
\end{eqnarray}
where we add up the norm over $K$ regions in the last step. For the second term on the right-hand-side of Eq.(\ref{m_2}), we use the inequality $<a,b>\le \frac{1}{2} \|a\|^2 +  \frac{1}{2} \|b\|^2 $ for any vectors $a,b$. Applying this inequality to the second term, we have
\begin{eqnarray}
& & -\sum_{i=1}^K \mathbb{E}\left< \nabla F^{(i)}(\theta_q) , \   \frac{1}{\Gamma_q^{(i)}} \sum_{n\in\mathcal{N}_q^{(i)}} \sum_{t=1}^T \gamma \left[ \nabla F_n^{(i)} (\theta_{q,n,t-1}) - \nabla F^{(i)}(\theta_q)\right] \right> \nonumber \\
& & \ \ \ \ \ = -\sum_{i=1}^K T\gamma \cdot \mathbb{E}\left< \nabla F^{(i)}(\theta_q) , \   \frac{1}{T\Gamma_q^{(i)}} \sum_{n\in\mathcal{N}_q^{(i)}} \sum_{t=1}^T  \left[ \nabla F_n^{(i)} (\theta_{q,n,t-1}) - \nabla F^{(i)}(\theta_q)\right] \right> \nonumber \\
& &  \ \ \ \ \ \le  \frac{T\gamma}{2} \sum_{i=1}^K  \mathbb{E} \left\| \nabla F^{(i)}(\theta_q) \right\|^2 + \frac{T\gamma}{2} \sum_{i=1}^K  \mathbb{E} \left\|  \frac{1}{T\Gamma_q^{(i)}} \sum_{n\in\mathcal{N}_q^{(i)}} \sum_{t=1}^T  \left[ \nabla F_n^{(i)} (\theta_{q,n,t-1}) - \nabla F^{(i)}(\theta_q)\right] \right\| \nonumber \\
& & \ \ \ \ \ = \frac{T\gamma}{2} \mathbb{E} \left\| \nabla F(\theta_q) \right\|^2 + \frac{T\gamma}{2} \left( \frac{L^2 \gamma^2 T NG}{\Gamma^*} + \frac{L^2\delta^2  N}{\Gamma^*}  \mathbb{E} \|\theta_{q} \|^2\right) \label{m_4}
\end{eqnarray}
where the second step uses the inequality and the third step follows directly from Lemma~2. Plugging Eq.(\ref{m_3}) and Eq.(\ref{m_4}) results into Eq.(\ref{m_2}), we obtain the desired upperbound:
\begin{eqnarray}
 \mathbb{E}\left< \nabla F(\theta_q) , \ \theta_{q+1} - \theta_{q} \right>  \le - \frac{T\gamma}{2} \mathbb{E} \left\| \nabla F(\theta_q) \right\|^2 + \frac{T\gamma}{2} \left( \frac{L^2 \gamma^2 T NG}{\Gamma^*} + \frac{L^2\delta^2  N}{\Gamma^*}  \mathbb{E} \|\theta_{q} \|^2\right). \label{mm_1}
\end{eqnarray}

\vspace{0.07in}
\noindent {\bf Upperbound for $\frac{L}{2 } \mathbb{E}\left\| \theta_{q+1} - \theta_{q} \right\|^2$}. We use the again result in Eq.(\ref{m_1}) and apply it to $\theta_{q+1} - \theta_{q}$, which gives:
\begin{eqnarray}
 & & \frac{L}{2 } \mathbb{E}\left\| \theta_{q+1} - \theta_{q} \right\|^2  \nonumber  \\ 
 & &  \ \ \ \ = \frac{L}{2 } \mathbb{E}\left\| \frac{1}{\Gamma_q^{(i)}} \sum_{n\in\mathcal{N}_q^{(i)}} \sum_{t=1}^T \gamma \nabla F_n^{(i)} (\theta_{q,n,t-1},\xi_{n,t-1})\right\|^2  \nonumber \\
& & \ \ \ \ \le \frac{3L}{2 } \mathbb{E}\left\| \frac{1}{\Gamma_q^{(i)}} \sum_{n\in\mathcal{N}_q^{(i)}} \sum_{t=1}^T \gamma  \left[ \nabla F_n^{(i)} (\theta_{q,n,t-1},\xi_{n,t-1}) - \nabla F_n^{(i)} (\theta_{q,n,t-1})\right]\right\|^2 \nonumber \\
& & \ \ \ \ \ \ \ \ + \frac{3L}{2} \mathbb{E}\left\| \frac{1}{\Gamma_q^{(i)}} \sum_{n\in\mathcal{N}_q^{(i)}} \sum_{t=1}^T \gamma  \left[ \nabla F_n^{(i)} (\theta_{q,n,t-1}) - \nabla F_n^{(i)} (\theta_{q})  \right] \right\|^2 \nonumber \\
& & \ \ \ \ \ \ \ \ + \frac{3L}{2} \mathbb{E}\left\| \frac{1}{\Gamma_q^{(i)}} \sum_{n\in\mathcal{N}_q^{(i)}} \sum_{t=1}^T \gamma \nabla F_n^{(i)} (\theta_{q})   \right\|^2, \label{mm_2}
\end{eqnarray}
where in the second step, we use the inequality $\|\sum_{i=1}^s a_i \|^2\le s \sum_{i=1}^s \|a_i \|^2 $ and split stochastic gradient $[\nabla F_n^{(i)} (\theta_{q,n,t-1},\xi_{n,t-1})]$ into $s=3$ parts, i.e., $[\nabla F_n^{(i)} (\theta_{q,n,t-1},\xi_{n,t-1}) - \nabla F_n^{(i)} (\theta_{q,n,t-1})]$, $[F_n^{(i)} (\theta_{q,n,t-1}) - F_n^{(i)} (\theta_{q})]$, and $[F_n^{(i)}(\theta_{q})]$. 

Next, we notice that the third term on the right-hand-side of Eq.(\ref{mm_2}) can be simplified, because (i) for IID data distribution, the cost function of each worker $n$ is the same as the global cost function, i.e., $\nabla F_n(\theta_q) = \nabla F(\theta_q)$, and (ii) for non-IID data distribution, the gradient noise assumption (Assumption~5) implies that $\frac{1}{\Gamma_q^{(i)}} \sum_{n\in\mathcal{N}_q^{(i)}} \nabla F_n (\theta_{q}) = F(\theta_{q})$. Thus in both cases, we have:
\begin{eqnarray}
& \frac{3L}{2} \mathbb{E}\left\| \frac{1}{\Gamma_q^{(i)}} \sum_{n\in\mathcal{N}_q^{(i)}} \sum_{t=1}^T \gamma \nabla F_n^{(i)} (\theta_{q})   \right\|^2 & \le \frac{3LT^2\gamma^2}{2}\sum_{i=1}^K \mathbb{E}\|  \nabla F^{(i)} (\theta_{q})  \|^2  \nonumber \\
&  & = \frac{3LT^2\gamma^2}{2} \mathbb{E}\| F (\theta_{q})  \|^2 , \label{mm_3}
\end{eqnarray}
where we again used the sum of norm of $K$ regions in the last step.

Now we notice that the first and second terms of Eq.(\ref{mm_2}) have been bounded by Lemma~2 and Lemma~3, excpet for constants $\gamma$ and ${1}/{T}$. Applying these results directly and also plugging in Eq.(\ref{mm_3}) into Eq.(\ref{mm_2}), we obtain the desired upperbound:
\begin{eqnarray}
 & \frac{L}{2 } \mathbb{E}\left\| \theta_{q+1} - \theta_{q} \right\|^2 & \le \frac{3LTN\gamma^2\sigma^2}{2({\Gamma^*})^2} {\rm \ (for \ IID) \ or \ }  \frac{3LTK\gamma^2\sigma^2}{2} {\rm \ (for \ non-IID)} \nonumber  \\ 
 & & \ \ \ \ + \frac{3L^3 \gamma^4 T^3 NG}{2\Gamma^*} + \frac{3L^3T^2\gamma^2\delta^2  N}{2\Gamma^*} \mathbb{E} \|\theta_{q} \|^2 \nonumber \\
 & & \ \ \ \ + \frac{3LT^2\gamma^2}{2} \mathbb{E}\| F_n (\theta_{q})  \|^2. \label{mm_4}
\end{eqnarray}

\vspace{0.07in}
\noindent {\bf Combining the two Upperbounds}. Finally, we will apply the upperbound for $\mathbb{E}\left< \nabla F(\theta_q) , \ \theta_{q+1} - \theta_{q} \right>$ in Eq.(\ref{mm_1}) as well as the upperbound for $\frac{L}{2} \mathbb{E}\left\| \theta_{q+1} - \theta_{q} \right\|^2$ in Eq.(\ref{mm_4}), and plug them into Eq.(\ref{mm_0}). First we take the sum over $q=1,\ldots,Q$ on both sides of Eq.(\ref{mm_0}), which becomes:
\begin{eqnarray}
 & & \mathbb{E} [ F(\theta_{Q+1})] - \mathbb{E} [ F(\theta_{0})] \nonumber \\
& &  \ \ \ \ \ \ \ \  = \sum_{q=1}^Q \mathbb{E} [ F(\theta_{q+1})] - \sum_{q=1}^Q \mathbb{E}[F(\theta_q)] \nonumber \\
 & &   \ \ \ \ \ \ \ \ \le \sum_{q=1}^Q \mathbb{E}\left< \nabla F(\theta_q) , \ \theta_{q+1} - \theta_{q} \right> + \sum_{q=1}^Q  \frac{L}{2 } \mathbb{E}\left\| \theta_{q+1} - \theta_{q} \right\|^2. \label{mm_5}
\end{eqnarray}
Now plugging in the two upperbounds and re-arranging the terms, for IID data distribution, we derive:
\begin{eqnarray}
 & & \mathbb{E} [ F(\theta_{Q+1})] - \mathbb{E} [ F(\theta_{0})] \nonumber \\
& &  \ \ \ \ \ \ \ \  \le - \frac{T\gamma}{2} \left( 1-3LT\gamma \right) \sum_{q=1}^Q \mathbb{E}\| \nabla F(\theta_q)\|^2  \nonumber \\ 
& & \ \ \ \ \ \ \ \  \ \ \ \ + \frac{\gamma TQ}{2} \left( \frac{TL^2\gamma^2 NG}{\Gamma^*} +\frac{3LN\gamma \sigma^2}{(\Gamma^*)^2} + \frac{3L^3\gamma^3T^3NG}{\Gamma^*} \right)   \nonumber \\
& & \ \ \ \ \ \ \ \  \ \ \ \ +\frac{T\gamma}{2} \left( \frac{L^2\delta^2 N}{\Gamma^*} +\frac{3L^3T\gamma \delta^2N}{\Gamma^*}  \right) \sum_{q=1}^T \mathbb{E} \| \theta_q \|^2.
\end{eqnarray}
We choose learning rate $\gamma\le 1/(6LT)$ and use the fact that $\mathbb{E} [ F(\theta_{Q+1})]$ is non-negative. The inequality above becomes: 
\begin{eqnarray}
& \frac{T\gamma}{4} \sum_{q=1}^Q \mathbb{E}\| \nabla F(\theta_q)\|^2 & \le \mathbb{E} [ F(\theta_{0})] + \frac{T\gamma Q}{2}\left( \frac{3LN\gamma \sigma^2}{(\Gamma^*)^2} +  \frac{3L^2\gamma^2TNG}{2\Gamma^*}\right) \nonumber \\
& & \ \ \ \ + \frac{T\gamma}{2}\left( \frac{3L^2\delta^2 N}{2\Gamma^*} \right)\sum_{q=1}^T \mathbb{E}  \| \theta_q \|^2.
\end{eqnarray}
Dividing both sides above by $4/(QT\gamma)$ and choosing $\gamma = 1/\sqrt{TQ}$, we have:
\begin{eqnarray}
& \frac{1}{Q} \sum_{q=1}^Q \mathbb{E}\| \nabla F(\theta_q)\|^2 & \le \frac{ 4\mathbb{E} [ F(\theta_{0})]}{\sqrt{TQ}} +  \frac{6LN \sigma^2}{\sqrt{TQ}(\Gamma^*)^2} +  \frac{3L^2NG}{Q\Gamma^*} \nonumber \\
& & \ \ \ \ + \frac{3L^2\delta^2 N}{\Gamma^*} \cdot \frac{1}{Q}\sum_{q=1}^T \mathbb{E}  \| \theta_q \|^2 \nonumber \\
& & = \frac{G_0}{\sqrt{TQ}} +  \frac{V_0}{\sqrt{Q}} + \frac{I_0}{\Gamma^*} \cdot \frac{1}{Q}\sum_{q=1}^T \mathbb{E}  \| \theta_q \|^2,
\end{eqnarray}
where we introduce constants $G_0=4\mathbb{E} [ F(\theta_{0})]+6LN \sigma^2/(\Gamma^*)^2$, $V_0=3L^2NG/\Gamma^*$, and $I_0=3L^2\delta^2 N$. This completes the proof of Theorem~1.

Finally, for non-IID data distribution, we plug the two upperbounds into Eq.(\ref{mm_5}) and re-arrange the terms. We follow a similar procedure and choose learning rate $\gamma=1/\sqrt{TQ}$ and $\gamma \le 1/(6LT)$. It is straightforward to show that for non-IID data distribution:
\begin{eqnarray}
\frac{1}{Q} \sum_{q=1}^Q \mathbb{E}\| \nabla F(\theta_q)\|^2 \le \frac{H_0}{\sqrt{TQ}} +  \frac{V_0}{\sqrt{Q}} + \frac{I_0}{\Gamma^*} \cdot \frac{1}{Q}\sum_{q=1}^T \mathbb{E}  \| \theta_q \|^2,
\end{eqnarray}
where $H_1=4\mathbb{E} [ F(\theta_{0})]+6LK\sigma^2$ is a different constant. This completes the proof of Theorem~2.

\section{Experimental Details}

\subsection{Experiment Setup}
The code implementation is open sourced and can be found at

\noindent\href{https://github.com/hanhanAnderson/FL_Converge}{\text{https://github.com/hanhanAnderson/FL\_Converge}}.

In this experimental section we evaluate different pruning techniques from state-of-the-art designs and verify our proposed theory under unifying pruning framework using two datasets.

Unless stated otherwise, the accuracy reported is defined as $$\frac{1}{n} \sum _{i} p_{i}\sum_{j}\text{Acc}(f_{i}(x_{j}^{(i)},\theta_{i}\odot m_{i}),y_{j}^{i}))$$ averaged over three random seeds with same random initialized starting $\theta_{0}$. Some key hyper-parameters includes total training rounds $Q = 100$,  local training epochs $T = 5$, testing batch size $bs = 128$ and local batch size $bl = 10$. Momentum for SGD is set to 0.5. standard batch normalization is used.

We focus on three points in our experiments: (i) the general coverage of federated learning with heterogeneous models by pruning (ii) the impact of coverage index $\Gamma_{min}$ (iii) the impact of mask error $\delta$.

We examine theoretical results on the following two common image classification datasets: MNIST and CIFAR10, among $N=100$ workers with IID and non-IID data with participation ratio $c = 0.1$. For IID data, we follow the design of balanced MNIST by previous research, and similarly obtain balanced CIFAR10. For non-IID data, we obtained balanced partition with label distribution skewed, where the number of the samples on each device is up to at most two out of ten possible classifications. 

\subsection{Pruning Techniques}
In the paper we select 4 pruning techniques as baselines and we elaborate the details of them.
Let $P_{m}=\frac{\|m\|_{0}}{|\theta|}$ be the sparsity of mask $m$, e.g.,$P_{m}=75\%$ for a model when 25 \% of its weights are pruned, and M is the number of the parameters in the model. 
Then a mask for weights pruning can be defined as:
\begin{equation}
    m_{i} = \begin{cases}
1 & \text{, if } \mathit{argsort}(\theta[i]) < P_{m} * M\\
0 & \text{, otherwise }
\end{cases} , i \in M
\end{equation}
Similarly we have the defination for neuron pruning:
\begin{equation}
m_{i} = \begin{cases}
1 & \text{, if } \mathit{argsort}(\sum \theta_{i}) < P_{m} * N\\
0 & \text{, otherwise }
\end{cases} , \theta_{i} \in \textbf{Neuron}\ i
\end{equation}

where N is the total number of neurons in the network, and fixed subnetwork:
\begin{equation}
  m_{i} = \begin{cases}
1 & \text{, if } i < P_{m} * M\\
0 & \text{, otherwise }
\end{cases} , i \in M  
\end{equation}
where M is the total number of parameters in the network.

Note in adaptive pruning such mask is subject to change after each round of global aggregation.
For pruning with pre-trained mask, the mask is generated based on eq(x) for first 3 rounds then fixed for the rest of the training.

An illustration of those pruning techniques can be found in figure.

\begin{figure}[h]
    \centering
    \includegraphics[width=0.4\textwidth]{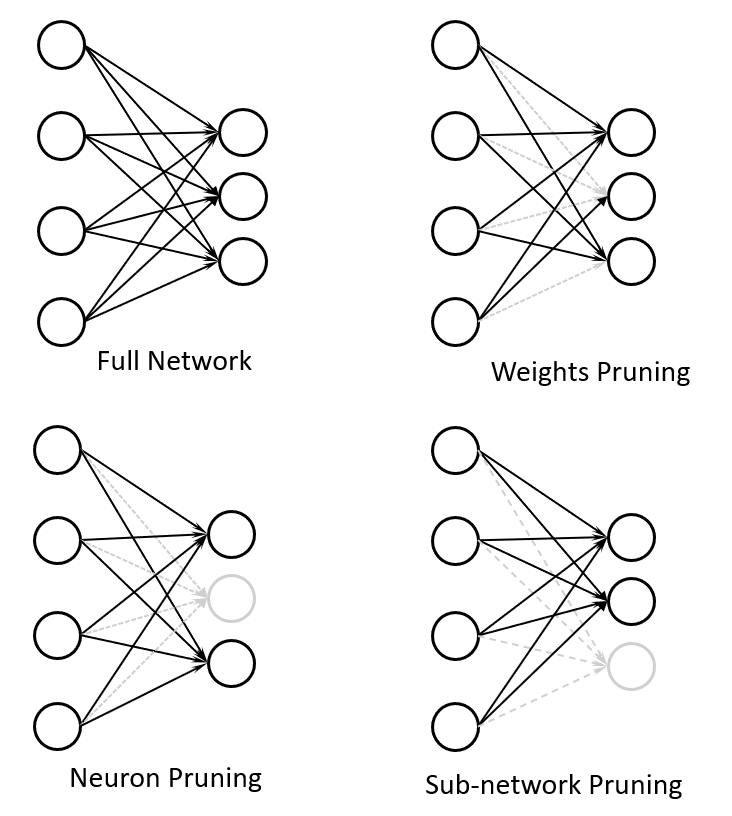}
    \caption{Illustration of pruning techniques used in this paper}
    \label{fig:my_label}
\end{figure}

\section{More Results on MNIST dataset}

In this section we present more supplementary experimental results on MNIST dataset. Specifically, we present the training progress in respect of global loss and accuracy for selected pruning techniques. For the final training results we focus on WP, FS and NP as PT is not found competitive without a carefully designed algorithm, however we still keep the training details for PT.
\subsection{Change of Notations}
In the main paper we use code name for simplicity of notation and better understanding. Here we present the results with their detailed settings.

For a full model without pruning it can be described as 
$\mathbb{P}_{1} (\theta)=  \{\textsl{S}_{1},\textsl{S}_{2},\textsl{S}_{3},\textsl{S}_{4}\}$, where
$$m_{i} = 1 \ \text{if} \ \theta_{i}\in \{\textsl{S}_{1}\cup \textsl{S}_{2}\cup \textsl{S}_{3}\cup \textsl{S}_{4}\} \ \text{ otherwise} \ m_{i} = 0$$.

Similarly we have another 3 pruning polices as follows:
$$\mathbb{P}_{2} (\theta)=  \{\textsl{S}_{1},\textsl{S}_{3},\textsl{S}_{4}\}$$
$$\mathbb{P}_{3} (\theta)=  \{\textsl{S}_{1},\textsl{S}_{2},\textsl{S}_{4}\}$$
$$\mathbb{P}_{4} (\theta)=  \{\textsl{S}_{1},\textsl{S}_{2},\textsl{S}_{3}\}$$

And we further denote a local client with its pruning policy, as an example, the case "*WP-M1" uses 4 local clients with full models, 2 local clients with pruned models using pruning policy  $\mathbb{P}_{4}$, 2 local clients with pruned models using pruning policy $\mathbb{P}_{2}$ and  2 local clients with pruned models using pruning policy $\mathbb{P}_{3}$, then we denote its code name as "1111223344" for simpler notation. Note that we continue to use code name "FedAvg" as a baseline rather than "1111111111". For the rest of the appendix we continue using such notations for denoting its pruning policy settings.

\begin{table}[h]
\centering
\resizebox{\textwidth}{!}{%
\begin{tabular}{@{}cccccccccccc@{}}
\toprule
\multirow{2}{*}{codename} &
  \multirow{2}{*}{1} &
  \multirow{2}{*}{0.75} &
  \multirow{2}{*}{0.5} &
  \multirow{2}{*}{PARAs} &
  \multirow{2}{*}{FLOPs} &
  \multirow{2}{*}{$\Gamma_{min}$} &
  \multirow{2}{*}{\%PARA} &
  \multirow{2}{*}{\%FLOPS} &
  IID &
  \multicolumn{2}{c}{Non-IID} \\ \cmidrule(l){10-12} 
           &    &   &   &         &         &    &          &          & Accuracy & Global & Local \\ \midrule
1111111111 & 10 &   &   & 159010 & 158800 & 10 & 1.00        & 1.00        & 98.045   & 93.59  & 93.82 \\
1111114444 & 6  & 4 &   & 143330 & 143120 & 6  & 0.90  & 0.90 & 98.18    & 95.15  & 95.49 \\
1111144447 & 5  & 4 & 1 & 135490 & 135280 & 5  & 0.85 & 0.85 & 97.51    & 89.13  & 89.29 \\
1111223344 & 4  & 6 &   & 135490 & 135280 & 8  & 0.85 & 0.85 & 98.325   & 95.48  & 95.82 \\
1111234444 & 4  & 6 &   & 135490 & 135280 & 6  & 0.85 & 0.85 & 98.395   & 95.45  & 95.96 \\
1111234567 & 4  & 3 & 3 & 123730 & 123520 & 7  & 0.77 & 0.77 & 96.735   & 88.99  & 88.9  \\
1111444444 & 4  & 6 &   & 135490 & 135280 & 4  & 0.85 & 0.85 & 97.85    & 89.13  & 89.29 \\
1111444477 & 4  & 4 & 2 & 127650 & 127440 & 4  & 0.80  & 0.80 & 96.99    & 93.02  & 93.12 \\
1111556677 & 4  &   & 6 & 111970 & 111760 & 6  & 0.70  & 0.70 & 95.545   & 80.07  & 79.34 \\
1114556677 & 3  & 1 & 6 & 108050 & 107840 & 5  & 0.67 & 0.67 & 95.80     & 79.3   & 79.75 \\
1234556677 & 1  & 3 & 6 & 100210 & 100000 & 5  & 0.63 & 0.62 & 95.315   & 81.66  & 81.64 \\
1455666777 & 1  & 1 & 8 & 92370  & 92160  & 3  & 0.58 & 0.58 & 94.795   & 79.15  & 79.08 \\
2233445677 & 0  & 6 & 4 & 104130 & 103920 & 5  & 0.65 & 0.65 & 95.955   & 81.27  & 81.17 \\
1444777777 & 1  & 3 & 6 & 92370  & 92160  & 6  & 0.65 & 0.65 & 95.10     & 72.19  & 71.64 \\ \bottomrule
\end{tabular}%
}
\caption{Results For Weights Pruning on MNIST}
\label{tab:my-table}
\end{table}

\begin{table}[h]
\centering
\resizebox{\textwidth}{!}{%
\begin{tabular}{@{}cccccccccccc@{}}
\toprule
\multirow{2}{*}{codename} &
  \multirow{2}{*}{100\%} &
  \multirow{2}{*}{75\%} &
  \multirow{2}{*}{50\%} &
  \multirow{2}{*}{PARAs} &
  \multirow{2}{*}{FLOPs} &
  \multirow{2}{*}{$\Gamma_{min}$} &
  \multirow{2}{*}{\%PARA} &
  \multirow{2}{*}{\%FLOPS} &
  IID &
  \multicolumn{2}{c}{Non-IID} \\ \cmidrule(l){10-12} 
           &    &   &   &         &         &    &          &       & Accuracy & Global & Local \\ \midrule
1111111111 & 10 &   &   & 159010 & 158800 & 10 & 1.00        & 1.00     & 98.13    & 95.31  & 95.33 \\
1111114444 & 6  & 4 &   & 143110 & 142920 & 6  & 0.90 & 0.90   & 97.97    & 93.6   & 93.82 \\
1111144447 & 5  & 4 & 1 & 135160 & 134980 & 5  & 0.85 & 0.85  & 97.395   & 91.92  & 92.18 \\
1111223344 & 4  & 6 &   & 135160 & 134980 & 8  & 0.85 & 0.85  & 97.865   & 91.9   & 92.42 \\
1111234444 & 4  & 6 &   & 135160 & 134980 & 6  & 0.85 & 0.85  & 97.86    & 92.99  & 92.93 \\
1111234567 & 4  & 3 & 3 & 123235 & 123070 & 7  & 0.77 & 0.775 & 96.645   & 83.82  & 83.61 \\
1111444444 & 4  & 6 &   & 135160 & 134980 & 4  & 0.85 & 0.85  & 97.53    & 91.8   & 92.07 \\
1111444477 & 4  & 4 & 2 & 127210 & 127040 & 4  & 0.80 & 0.80   & 96.775   & 84.91  & 85.02 \\
1111556677 & 4  &   & 6 & 111310 & 111160 & 6  & 0.70 & 0.70   & 96.575   & 69.11  & 69.63 \\
1114556677 & 3  & 1 & 6 & 107335 & 107190 & 5  & 0.67 & 0.675 & 95.345   & 77.53  & 77.7  \\
1234556677 & 1  & 3 & 6 & 99385  & 99250  & 5  & 0.62 & 0.625 & 95.475   & 72.8   & 72.4  \\
1455666777 & 1  & 1 & 8 & 91435  & 91310  & 3  & 0.57 & 0.575 & 94.41    & 61.96  & 62.49 \\
2233445677 & 0  & 6 & 4 & 103360 & 103220 & 5  & 0.65 & 0.65  & 96.375   & 60.23  & 61.01 \\
1444777777 & 1  & 3 & 6 & 99385  & 99250  & 5  & 0.62 & 0.625 & 95.23    & 60.54  & 61.85 \\ \midrule
           &    &   &   &         &         &    &          &       &          &        &      
\end{tabular}%
}
\caption{Results For Neuron Pruning on MNIST}
\label{tab:my-table}
\end{table}

\begin{table}[h]
\centering
\resizebox{\textwidth}{!}{%
\begin{tabular}{@{}cccccccccccc@{}}
\toprule
\multirow{2}{*}{codename} &
  \multirow{2}{*}{100\%} &
  \multirow{2}{*}{75\%} &
  \multirow{2}{*}{50\%} &
  \multirow{2}{*}{PARAs} &
  \multirow{2}{*}{FLOPs} &
  \multirow{2}{*}{$\Gamma_{min}$} &
  \multirow{2}{*}{\%PARA} &
  \multirow{2}{*}{\%FLOPS} &
  IID &
  \multicolumn{2}{c}{Non-IID} \\ \cmidrule(l){10-12} 
           &    &   &   &         &         &    &          &       & Accuracy & Global & Local \\ \midrule
1111111111 & 10 &   &   & 159010 & 158800 & 10 & 1.00        & 1.00     & 97.67    & 94.12  & 94.45 \\
1111114444 & 6  & 4 &   & 143110 & 142920 & 6  & 0.9      & 0.90  & 97.76    & 92.33  & 92.55 \\
1111144447 & 5  & 4 & 1 & 135160 & 134980 & 6  & 0.85     & 0.85  & 97.34    & 93.79  & 93.92 \\
1111444444 & 4  & 6 &   & 135160 & 134980 & 4  & 0.85     & 0.85  & 97.62    & 92.05  & 92.33 \\
1111444477 & 4  & 4 & 2 & 127210 & 127040 & 4  & 0.80 & 0.8   & 97.32    & 92.67  & 92.95 \\
1111444777 & 4  & 3 & 3 & 123235 & 123070 & 4  & 0.77 & 0.775 & 97.35    & 91.34  & 91.73 \\
1111777777 & 4  &   & 6 & 111310 & 111160 & 4  & 0.70 & 0.7   & 97.18    & 93.6   & 93.48 \\
1114777777 & 3  & 1 & 6 & 107335 & 107190 & 3  & 0.67  & 0.675 & 97.12    & 93.7   & 93.57 \\
1444777777 & 1  & 3 & 6 & 99385  & 99250  & 1  & 0.62 & 0.625 & 97.01    & 90.74  & 90.57 \\
1477777777 & 1  & 1 & 8 & 91435  & 91310  & 1  & 0.57 & 0.575 & 96.88    & 90.73  & 90.67 \\ \bottomrule
\end{tabular}%
}
\caption{Results For Fixed Sub-network on MNIST}
\label{tab:my-table}
\end{table}

\subsection{More Results}

\subsubsection{Case for IID data}
We present the full results of training  for IID case in Fig 2 - 5
\\

\begin{figure}[h]
     \centering
     \begin{subfigure}[b]{0.49\textwidth}
         \centering
         \includegraphics[width=\textwidth]{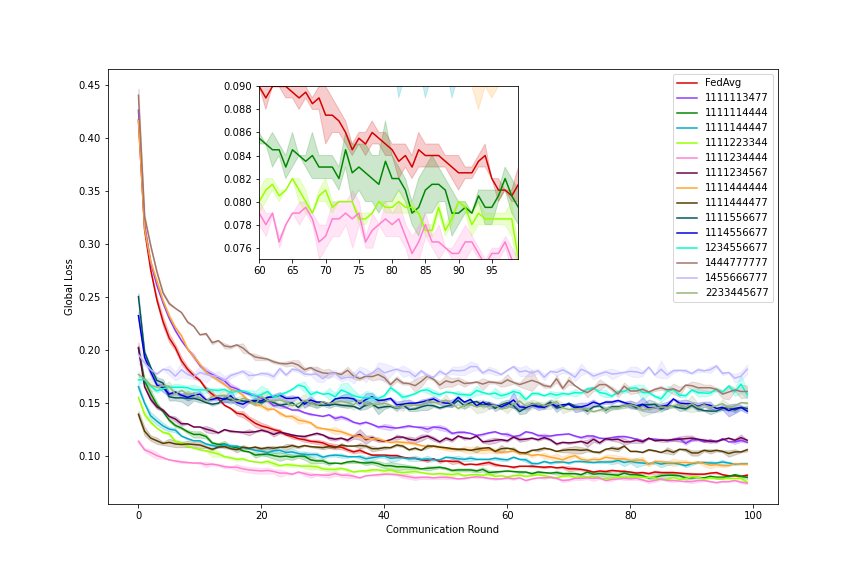}
         \caption{Global Loss}
         \label{fig:1}
     \end{subfigure}
     \hfill
     \begin{subfigure}[b]{0.49\textwidth}
         \centering
         \includegraphics[width=\textwidth]{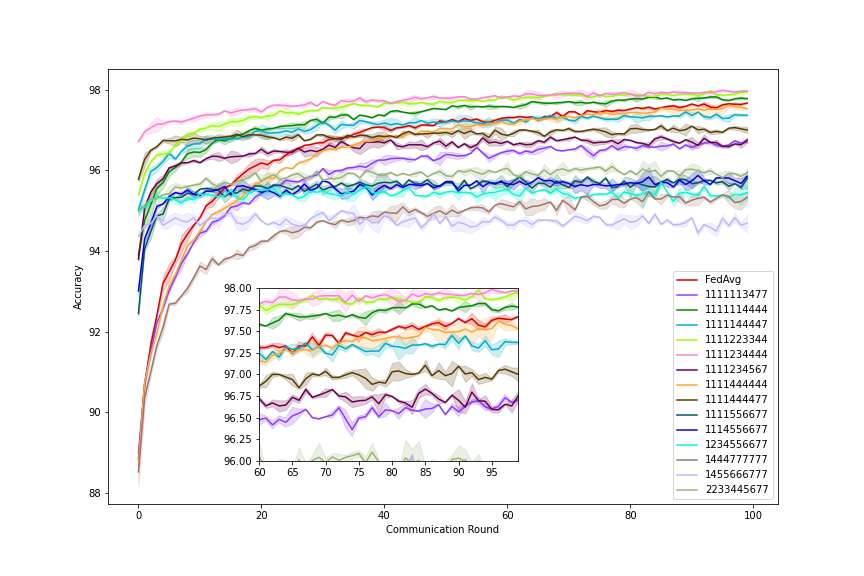}
         \caption{Accuracy}
         \label{fig:three sin x}
     \end{subfigure}
        \caption{Results on Weights Pruning on MNIST IID}
        \label{fig:1}
\end{figure}

\begin{figure}[h]
     \centering
     \begin{subfigure}[b]{0.49\textwidth}
         \centering
         \includegraphics[width=\textwidth]{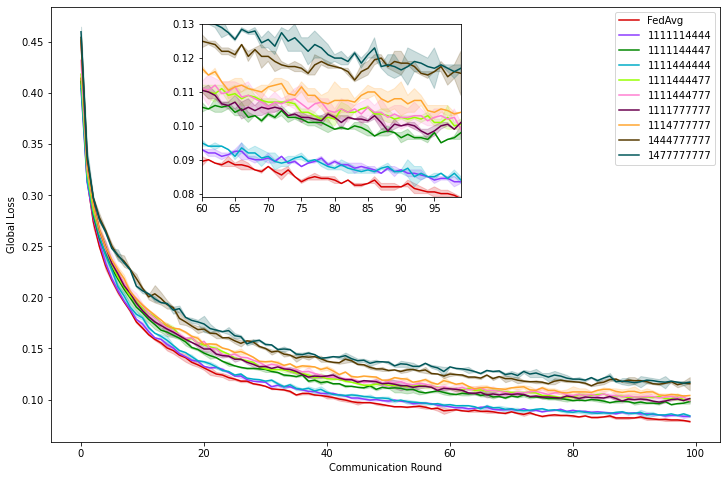}
         \caption{Global Loss}
         \label{fig:1}
     \end{subfigure}
     \hfill
     \begin{subfigure}[b]{0.49\textwidth}
         \centering
         \includegraphics[width=\textwidth]{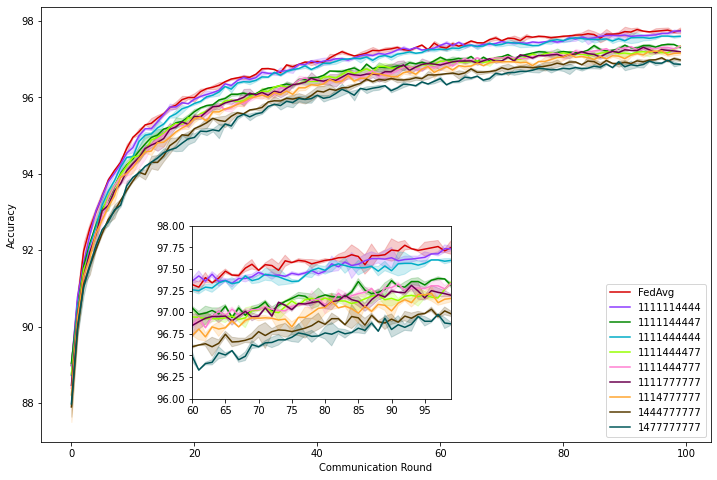}
         \caption{Accuracy}
         \label{fig:three sin x}
     \end{subfigure}
        \caption{Results on Fixed Sub-network on MNIST IID}
        \label{fig:1}
\end{figure}

\begin{figure}[h]
     \centering
     \begin{subfigure}[b]{0.49\textwidth}
         \centering
         \includegraphics[width=\textwidth]{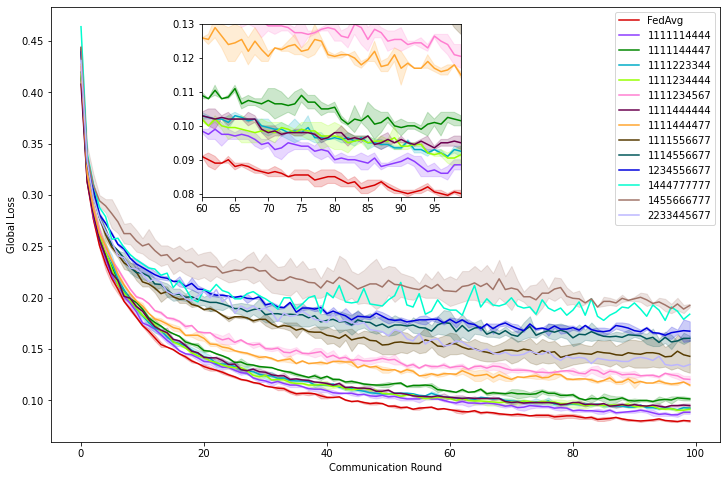}
         \caption{Global Loss}
         \label{fig:1}
     \end{subfigure}
     \hfill
     \begin{subfigure}[b]{0.49\textwidth}
         \centering
         \includegraphics[width=\textwidth]{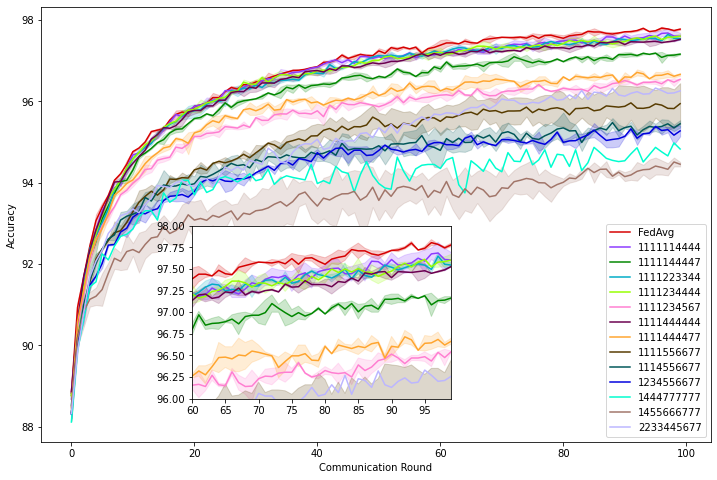}
         \caption{Accuracy}
         \label{fig:three sin x}
     \end{subfigure}
        \caption{Results on Neuron Pruning on MNIST IID}
        \label{fig:1}
\end{figure}

\begin{figure}[h]
     \centering
     \begin{subfigure}[b]{0.49\textwidth}
         \centering
         \includegraphics[width=\textwidth]{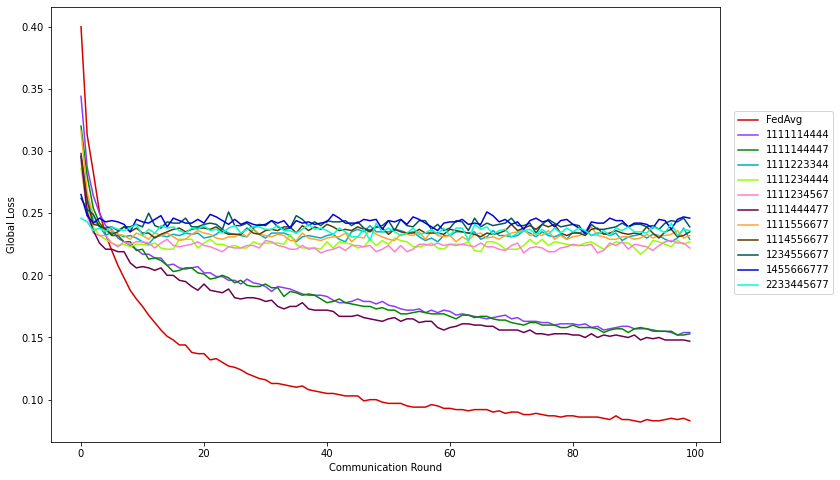}
         \caption{Global Loss}
         \label{fig:1}
     \end{subfigure}
     \hfill
     \begin{subfigure}[b]{0.49\textwidth}
         \centering
         \includegraphics[width=\textwidth]{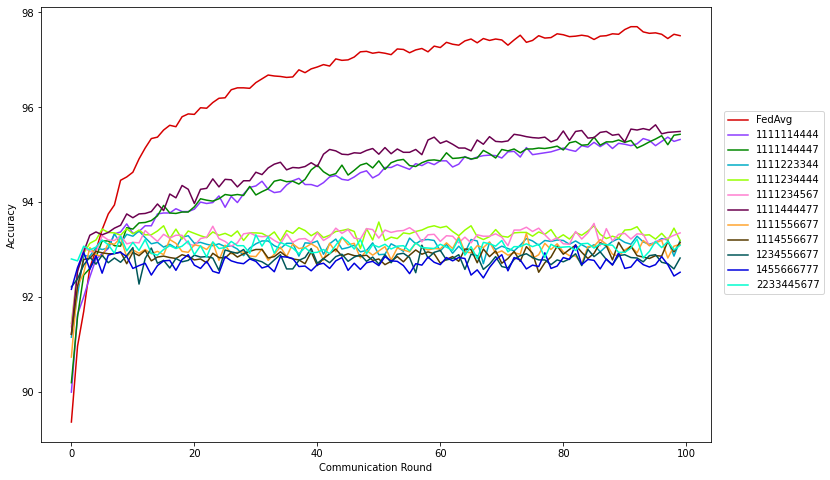}
         \caption{Accuracy}
         \label{fig:three sin x}
     \end{subfigure}
        \caption{Results on Pruning with pre-trained mask on MNIST IID}
        \label{fig:1}
\end{figure}





\subsubsection{Case for non-IID data}

We present the full results of training  for non-IID case in Fig 6 - 9

\begin{figure}[h]
     \centering
     \begin{subfigure}[b]{0.49\textwidth}
         \centering
         \includegraphics[width=\textwidth]{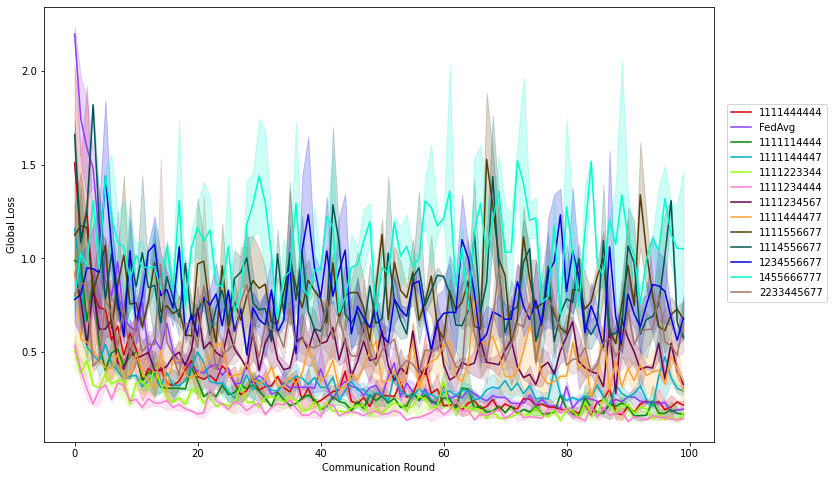}
         \caption{Global Loss}
         \label{fig:1}
     \end{subfigure}
     \hfill
     \begin{subfigure}[b]{0.49\textwidth}
         \centering
         \includegraphics[width=\textwidth]{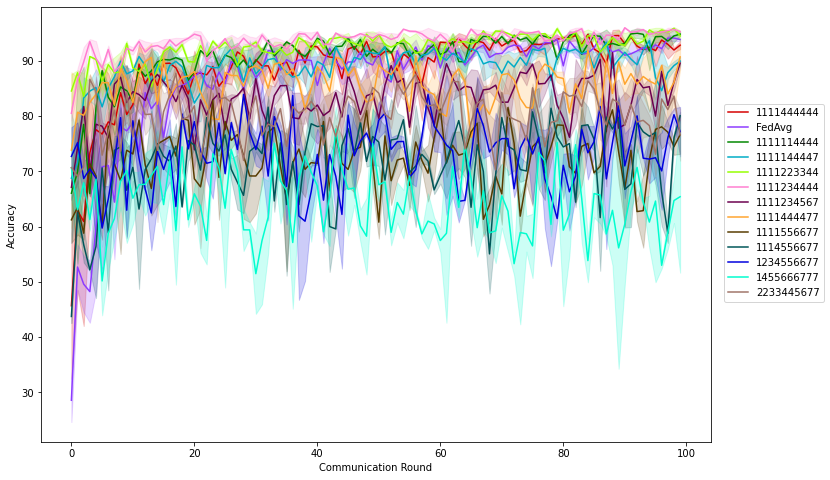}
         \caption{Accuracy}
         \label{fig:three sin x}
     \end{subfigure}
        \caption{Results on Weights Pruning on MNIST non-IID}
        \label{fig:1}
\end{figure}

\begin{figure}[h]
     \centering
     \begin{subfigure}[b]{0.49\textwidth}
         \centering
         \includegraphics[width=\textwidth]{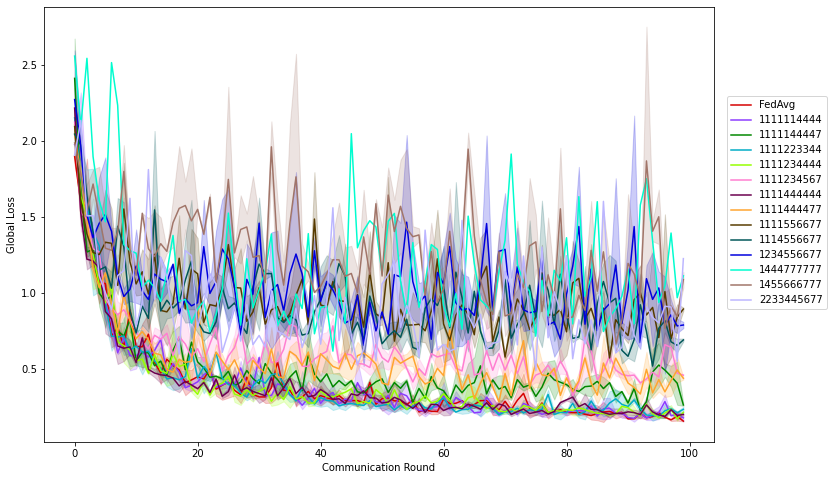}
         \caption{Global Loss}
         \label{fig:1}
     \end{subfigure}
     \hfill
     \begin{subfigure}[b]{0.49\textwidth}
         \centering
         \includegraphics[width=\textwidth]{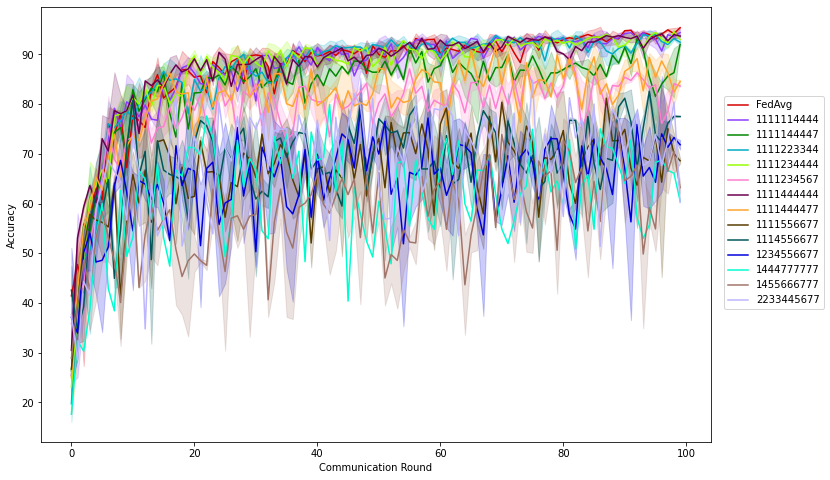}
         \caption{Accuracy}
         \label{fig:three sin x}
     \end{subfigure}
        \caption{Results on Neuron Pruning on MNIST non-IID}
        \label{fig:1}
\end{figure}

\begin{figure}[h]
     \centering
     \begin{subfigure}[b]{0.49\textwidth}
         \centering
         \includegraphics[width=\textwidth]{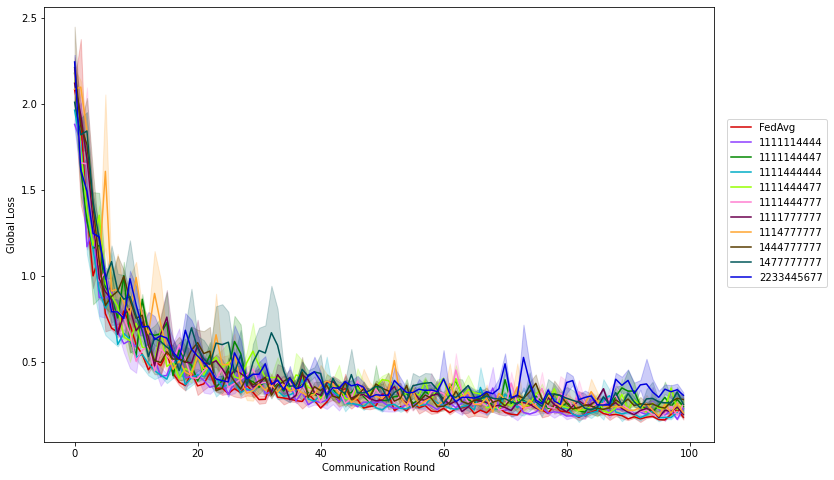}
         \caption{Global Loss}
         \label{fig:1}
     \end{subfigure}
     \hfill
     \begin{subfigure}[b]{0.49\textwidth}
         \centering
         \includegraphics[width=\textwidth]{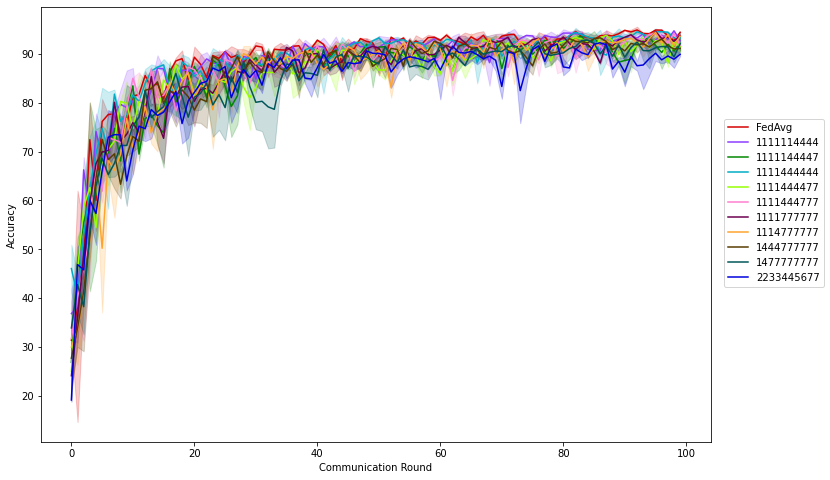}
         \caption{Accuracy}
         \label{fig:three sin x}
     \end{subfigure}
        \caption{Results on Fixed Sub-network on MNIST non-IID}
        \label{fig:1}
\end{figure}

\begin{figure}[h]
     \centering
     \begin{subfigure}[b]{0.49\textwidth}
         \centering
         \includegraphics[width=\textwidth]{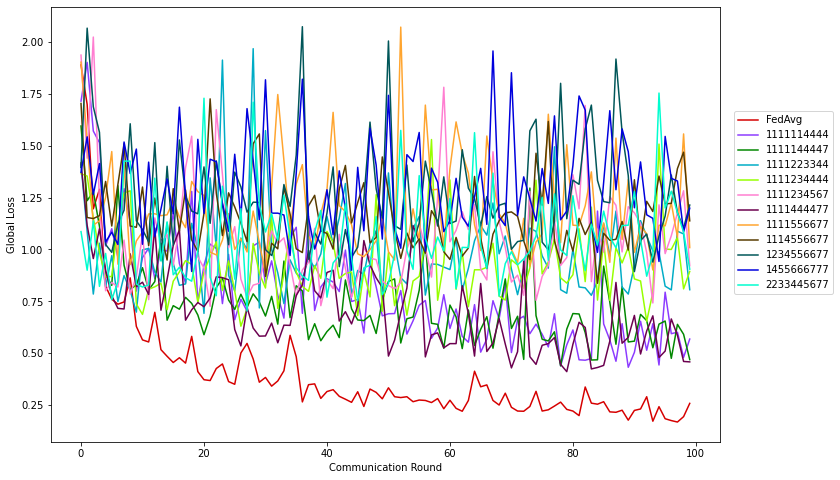}
         \caption{Global Loss}
         \label{fig:1}
     \end{subfigure}
     \hfill
     \begin{subfigure}[b]{0.49\textwidth}
         \centering
         \includegraphics[width=\textwidth]{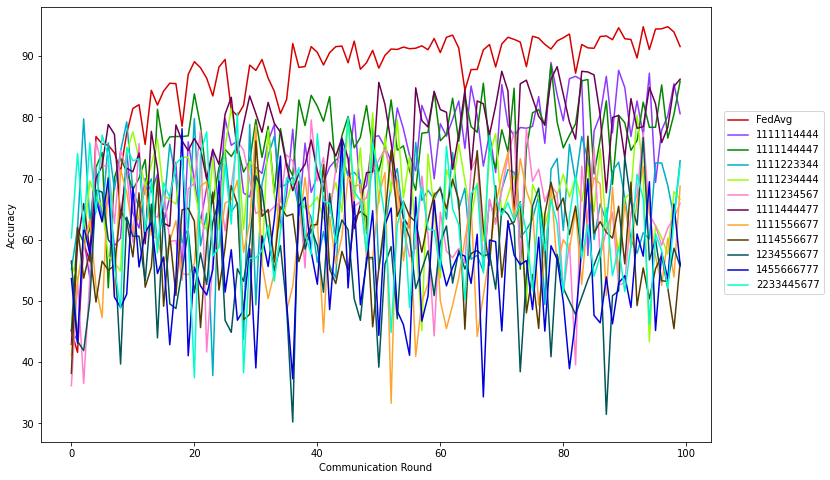}
         \caption{Accuracy}
         \label{fig:three sin x}
     \end{subfigure}
        \caption{Results on Pruning with pre-trained mask on MNIST non-IID}
        \label{fig:1}
\end{figure}





\section{More Results For CIFAR-10-IID}

In this section we present more supplementary experimental results on CIFAR 10 dataset to test the effects of pruning on convolutional layers. Specifically, we present the training progress in respect of global loss and accuracy for selected pruning techniques where we focus on WP and FS. 

\subsection{Change of Notations}
In the main paper we use code name for simplicity of notation and better understanding. Here we present the results with their detailed settings.

For a full model without pruning it can be described as 
$\mathbb{P}_{1} (\theta)=  \{\textsl{S}_{1},\textsl{S}_{2},\textsl{S}_{3},\textsl{S}_{4}\}$, where
$$m_{i} = 1 \ \text{if} \ \theta_{i}\in \{\textsl{S}_{1}\cup \textsl{S}_{2}\cup \textsl{S}_{3}\cup \textsl{S}_{4}\} \ \text{ otherwise} \ m_{i} = 0$$.
As we have demonstrated the effects of pruning MLP layers, on CIFAR10 datasets we focus on the effects of conv2d layers.

We have another 3 pruning polices for conv2d layers as follows:
$$\mathbb{P}_{2} (\theta)=  \{\textsl{S}_{1},\textsl{S}_{3},\textsl{S}_{4}\}$$
$$\mathbb{P}_{3} (\theta)=  \{\textsl{S}_{1},\textsl{S}_{2},\textsl{S}_{4}\}$$
$$\mathbb{P}_{4} (\theta)=  \{\textsl{S}_{1},\textsl{S}_{2},\textsl{S}_{3}\}$$

For WP and PT, when using $\mathbb{P}_{2}$ the top 75\% of kernels will be kept,  i.e. for the first conv2d layer, the 5 largest kernels out of total of 6 kernels will be kept, and the 6-th kernel will be pruned. Under all pruning polices MLP layers will be pruned at 75\% accordingly. Note under such settings, code name without full model '1' , e.g. '2222333444', will not satisfy our necessary condition of convergence.

For FS, we denote $\mathbb{P}_{2}$ as the similar policy as above but only the first continuous parameters, i.e. for the first conv2d layer, the first 5 kernels out of total of 6 kernels will be kept, and the 6-th kernel will be pruned, together with pruning MLP layers at 75\%. We denote $\mathbb{P}_{3}$ as only pruning conv2d layers and $\mathbb{P}_{4}$ as only pruning MLP layers. In this case, note that even with same codename for WP and FS, their results are NOT directly comparable.

And we further denote a local client with its pruning policy, as an example, the case "*WP-M1" uses 4 local clients with full models, 2 local clients with pruned models using pruning policy  $\mathbb{P}_{4}$, 2 local clients with pruned models using pruning policy $\mathbb{P}_{2}$ and  2 local clients with pruned models using pruning policy $\mathbb{P}_{3}$, then we denote its code name as "1111223344" for simpler notation. Note that we continue to use code name "FedAvg" as a baseline rather than "1111111111". For the rest of the appendix we continue using such notations for denoting its pruning policy settings. For the final training results we focus on WP, FS and NP as PT is not found competitive without a carefully designed algorithm, however we still keep the training details for PT.

\begin{table}[]
\centering
\resizebox{\textwidth}{!}{%
\begin{tabular}{@{}llcccc@{}}
\toprule
Codename   & PARAs(K)        & \%   & FLOPs(K) & \%   & Testing Accuracy \\ \midrule
1111111111 & 512.80           & 1.00 & 653.8    & 1.00 & 53.63            \\
1111111122 & 482.34          & 0.94 & 619.6    & 0.94 & 53.12            \\
1111112222 & 451.936         & 0.88 & 587.0    & 0.89 & 52.66            \\
1111112223 & 451.936         & 0.88 & 587.0    & 0.89 & 52.98            \\
1111112233 & 451.936         & 0.88 & 587.0    & 0.89 & 54.20            \\
1111113333 & 451.936         & 0.88 & 587.0    & 0.89 & 52.96            \\
1111114444 & 451.936         & 0.88 & 587.0    & 0.89 & 51.61            \\
1111222222 & 421.504         & 0.82 & 553.7    & 0.84 & 51.69            \\
1111222334 & 421.504         & 0.82 & 553.7    & 0.84 & 52.20            \\
1111223344 & 421.504         & 0.82 & 553.7    & 0.84 & 52.54            \\
1222333444 & 375.856         & 0.73 & 503.6    & 0.77 & 49.15            \\ \bottomrule
\end{tabular}%
}
\caption{Results For Weights Pruning on CIFAR 10}
\label{tab:my-table}
\end{table}

\begin{table}[]
\centering
\resizebox{\textwidth}{!}{%
\begin{tabular}{@{}llcccc@{}}
\toprule
Codename   & PARAs(K) & \%   & FLOPs(K) & \%   & Testing Accuracy \\ \midrule
1111111111 & 512.81   & 1.00 & 653.80   & 1.00 & 54.78            \\
1111111122 & 476.37   & 0.92 & 619.68   & 0.94 & 54.10             \\
1111112222 & 439.93   & 0.85 & 585.57   & 0.89 & 52.87            \\
1111113333 & 471.28   & 0.91 & 589.48   & 0.90 & 53.96            \\
1111113344 & 467.92   & 0.91 & 589.06   & 0.90 & 53.90             \\
1111114444 & 464.57   & 0.90 & 588.64   & 0.90 & 54.44            \\
1111222222 & 403.49   & 0.78 & 551.46   & 0.84 & 52.74            \\
2222333444 & 372.59   & 0.72 & 488.47   & 0.74 & 52.35            \\ \bottomrule
\end{tabular}%
}
\caption{Results For Fixed Sub-network on CIFAR 10}
\label{tab:my-table}
\end{table}

\begin{figure}[h]
     \centering
     \begin{subfigure}[b]{0.49\textwidth}
         \centering
         \includegraphics[width=\textwidth]{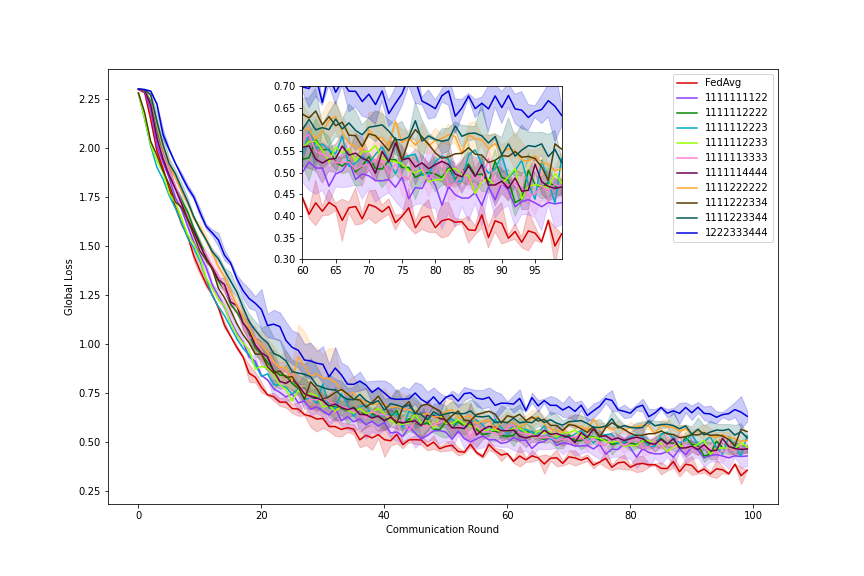}
         \caption{Global Loss}
         \label{fig:1}
     \end{subfigure}
     \hfill
     \begin{subfigure}[b]{0.49\textwidth}
         \centering
         \includegraphics[width=\textwidth]{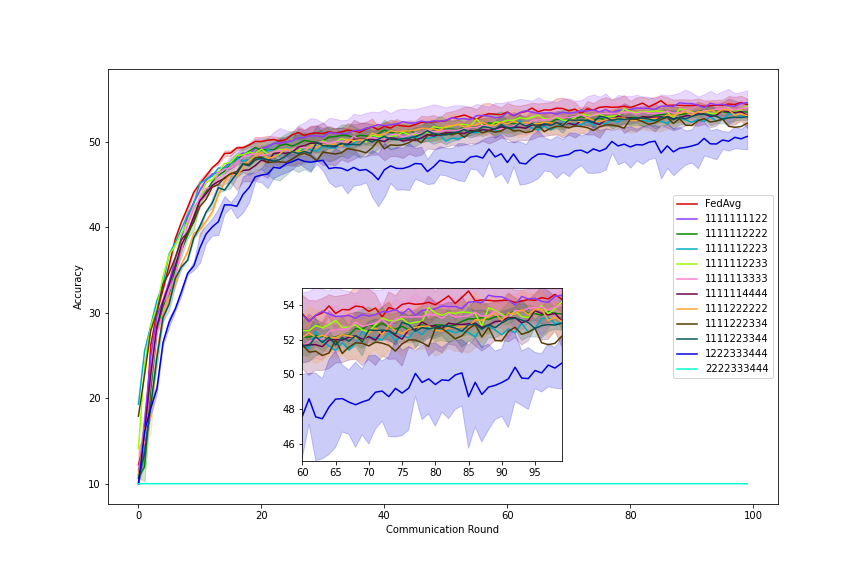}
         \caption{Accuracy}
         \label{fig:three sin x}
     \end{subfigure}
        \caption{Results on Weights Pruning on CIFAR10 IID}
        \label{fig:1}
\end{figure}

\begin{figure}[h]
     \centering
     \begin{subfigure}[b]{0.49\textwidth}
         \centering
         \includegraphics[width=\textwidth]{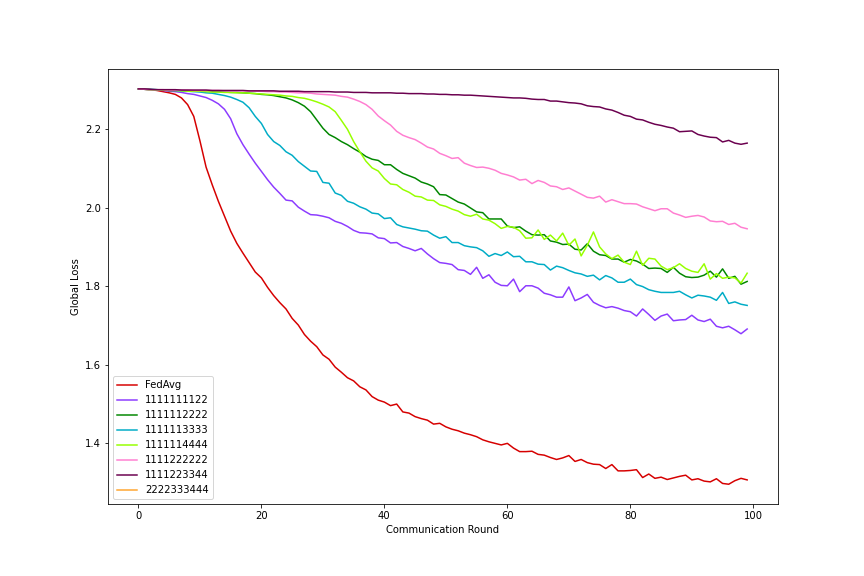}
         \caption{Global Loss}
         \label{fig:1}
     \end{subfigure}
     \hfill
     \begin{subfigure}[b]{0.49\textwidth}
         \centering
         \includegraphics[width=\textwidth]{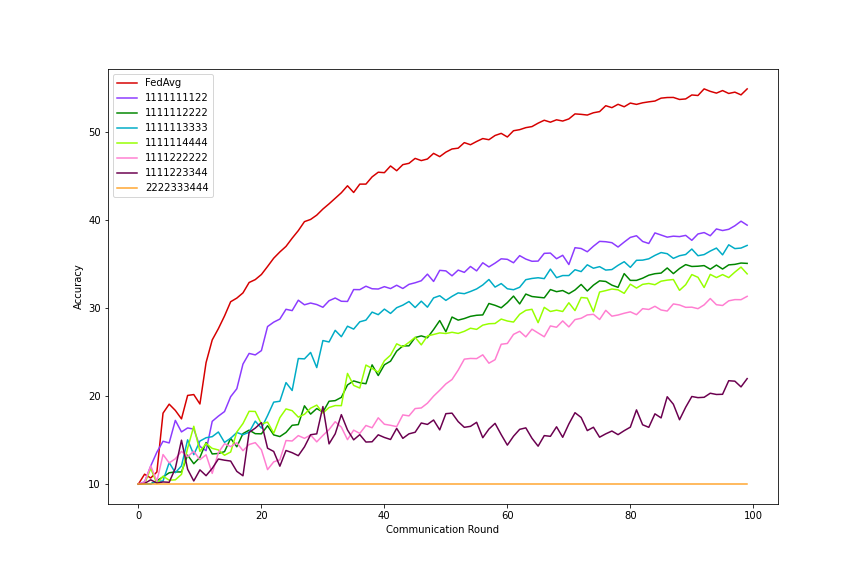}
         \caption{Accuracy}
         \label{fig:three sin x}
     \end{subfigure}
        \caption{Results on Pruning with pre-trained mask on CIFAR10 IID}
        \label{fig:1}
\end{figure}

\begin{figure}[h]
     \centering
     \begin{subfigure}[b]{0.49\textwidth}
         \centering
         \includegraphics[width=\textwidth]{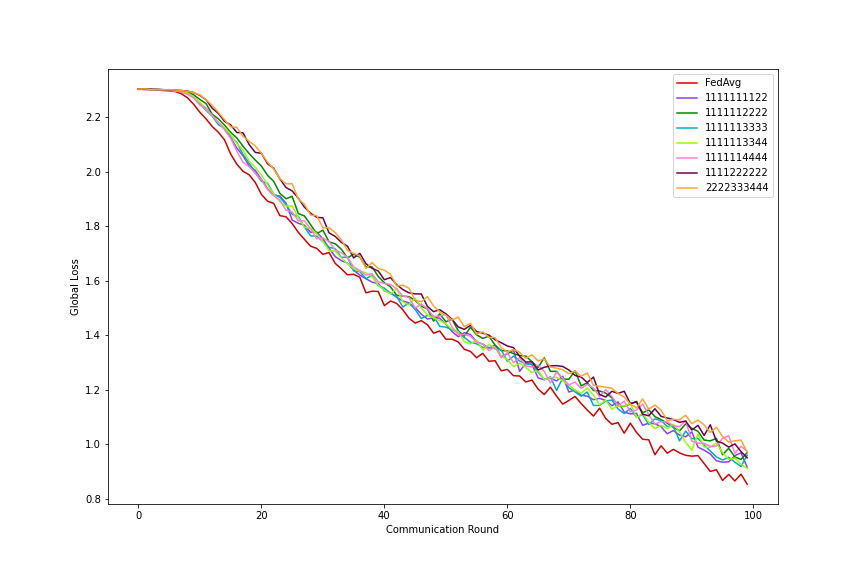}
         \caption{Global Loss}
         \label{fig:1}
     \end{subfigure}
     \hfill
     \begin{subfigure}[b]{0.49\textwidth}
         \centering
         \includegraphics[width=\textwidth]{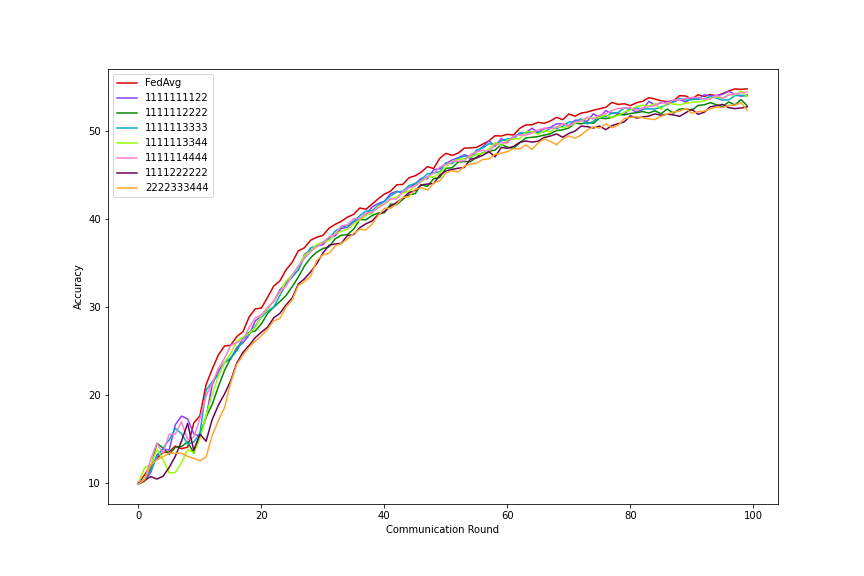}
         \caption{Accuracy}
         \label{fig:three sin x}
     \end{subfigure}
        \caption{Results on Fixed Sub-network Pruning on CIFAR10 IID}
        \label{fig:1}
\end{figure}

\end{appendices}

\end{document}